\documentclass[journal]{IEEEtran}
\usepackage{graphicx}
\usepackage{fancyhdr, ifpdf}
\usepackage{amssymb}
\usepackage{amsmath}
\usepackage{amsthm}
\usepackage{mathrsfs}
\usepackage{color}
\usepackage{enumerate}
\usepackage{graphics}
\usepackage{epsfig}
\usepackage{psfrag}
\usepackage{amsfonts}
\usepackage{mathrsfs}
\usepackage{verbatim}
\usepackage{cite}
\usepackage[numbers,sort&compress]{natbib}
\usepackage{color}
\usepackage{booktabs}
\usepackage{epstopdf}
\usepackage[tight,footnotesize]{subfigure}
\usepackage{algorithm}
\usepackage{algorithmicx}
\usepackage{algpseudocode}
\usepackage{subfigure}
\usepackage{multicol}
\usepackage{multirow}
\usepackage{bbding}
\usepackage{caption}
\usepackage{lineno,hyperref}
\usepackage{paralist}
\usepackage[OT1]{fontenc}
\usepackage{indentfirst}
\ifCLASSOPTIONcompsoc
  \usepackage[nocompress]{cite}
\else
  \usepackage{cite}
\fi
\begin{document}

\title{Anchor Graph Structure Fusion Hashing for Cross-Modal Similarity Search}

\author{Lu Wang, Jie Yang$^*$, Masoumeh Zareapoor, Zhonglong Zheng$^*$

\thanks{$*$Corresponding author: Jie Yang, Zhonglong Zheng (jieyang@sjtu.edu.cn; zhonglong@zjnu.cn). }
\thanks{L. Wang (luwang\_16@sjtu.edu.cn),  J. Yang, M. Zareapoor are with the Institute of Image Processing and Pattern Recognition, Department of Automation, Shanghai Jiao Tong University, Shanghai, China, 201100. Z. Zheng is with Zhejiang Normal University, China.}}

\markboth{}
{Lu \MakeLowercase{\textit{et al.}}: Anchor Graph Structure Fusion Hashing for Cross-Modal Similarity Search}
\maketitle

\begin{abstract}
Cross-modal hashing has been widely applied to retrieve items across modalities due to its superiority in fast computation and low storage. However, some challenges are still needed to address: (1) most existing CMH methods take graphs, which are always predefined separately in each modality, as input to model data distribution. These methods omit to consider the correlation of graph structure among multiple modalities. Besides, cross-modal retrieval results highly rely on the quality of predefined affinity graphs; (2) most existing CMH methods deal with the preservation of intra- and inter-modal affinity independently to learn the binary codes, which ignores considering the fusion affinity among multi-modalities data; (3) most existing CMH methods relax the discrete constraints to solve the optimization objective, which could significantly degrade the retrieval performance. To solve the above limitations, in this paper, we propose a novel Anchor Graph Structure Fusion Hashing (AGSFH). AGSFH constructs the anchor graph structure fusion matrix from different anchor graphs of multiple modalities with the Hadamard product, which can fully exploit the geometric property of underlying data structure across multiple modalities. Specifically, based on the anchor graph structure fusion matrix, AGSFH makes an attempt to directly learn an intrinsic anchor graph, where the structure of the intrinsic anchor graph is adaptively tuned so that the number of components of the intrinsic graph is exactly equal to the number of clusters. Based on this process, training instances can be clustered into semantic space. Besides, AGSFH preserves the anchor fusion affinity into the common binary Hamming space, capturing intrinsic similarity and structure across modalities by hash codes. Furthermore, a discrete optimization framework is designed to learn the unified binary codes across modalities. Extensive experimental results on three public social datasets demonstrate the superiority of AGSFH in cross-modal retrieval.
\end{abstract}


\IEEEpeerreviewmaketitle

\section{Introduction}
\IEEEPARstart{I}{n} the wake of the rapid growth of multiple modalities data (e.g., image, text, and video) on social media and the internet, cross-modal information retrieval have obtained much attention for storing and searching items in the large-scale data environment. Hashing based methods \cite{proceeding7, proceeding10, proceeding27, jour13, yan2020deep, yan2021task, yan2021precise, yan2021age, ma2018global, ma2020discriminative, ma2021learning, ma2020correlation, ma2017manifold, ma2017learning} has shown their superiority in the approximate nearest neighbor retrieval task because of their low storage consumption and fast search speed. Its superiority is achieved by compact binary codes representation, which targets constructing the compact binary representation in Hamming space to preserve the semantic affinity information across modalities as much as possible.

By whether the label information is utilized, cross-modal hashing (CMH) methods can be categorized into two subclasses: unsupervised \cite{proceeding17, proceeding14, proceeding15, proceeding16, jour5, jour23} and supervised \cite{proceeding18, proceeding19, proceeding20, proceeding21, proceeding22} methods. The details are shown in Section \ref{Sec2}. Besides, most existing cross-modal hashing methods separately preserve the intra- and inter-modal affinities to learn the corresponding binary codes, which omit to consider the fusion semantic affinity among multi-modalities data. However, such fusion affinity is the key to capturing the cross-modal semantic similarity since different modalities' data could complement each other. Recently, some works have considered utilizing multiple similarities \cite{proceeding35, jour24}. Few works have been proposed to capture such fusion semantic similarity with binary codes. The symbolic achievement is Fusion Similarity Hashing (FSH) \cite{proceeding17}. FSH builds an asymmetrical fusion graph to capture the intrinsic relations across heterogeneous data and directly preserves the fusion similarity to a Hamming space.

However, there are some limitations to be further addressed. Firstly, most existing CMH methods take graphs, which are always predefined separately in each modality, as input to model the distribution of data. These methods omit to consider the correlation of graph structure among multiple modalities, and cross-modal retrieval results highly rely on the quality of predefined affinity graphs \cite{proceeding17, proceeding14, proceeding16}. Secondly, most existing CMH methods deal with the preservation of intra- and inter-modal affinity separately to learn the binary code, omitting to consider the fusion affinity among multi-modalities data containing complementary information \cite{proceeding14, proceeding15, proceeding16}. Thirdly, most existing CMH methods relax the discrete constraints to solve the optimization objective, which could significantly degrade the retrieval performance \cite{proceeding14, proceeding18, proceeding19}. Besides, there are also some works \cite{quan2016object, yan2020depth} that try to adaptively learn the optimal graph affinity matrix from the data information and the learning tasks. Graph Optimized-Flexible Manifold Ranking (GO-FMR) \cite{quan2016object} obtains the optimal affinity matrix via direct $L_2$ regression under the guidance of the human established affinity matrix to infer the final predict labels more accurately. In our model, we learn an intrinsic anchor graph also via direct $L_2$ regression like \cite{quan2016object} based on the constructed anchor graph structure fusion matrix, but it also simultaneously tunes the structure of the intrinsic anchor graph to make the learned graph has exactly $C$ connected components. Hence, vertices in each connected component of the graph can be categorized into one cluster. In addition, GO-FMR needs $O(N^2)$ in computational complexity to learn the optimal affinity matrix ($N$ is the number of the vertices in the graph). With the increasing amount of data, it is intractable for off-line learning. Compared with GO-FMR, the proposed method only needs $O(PN)$ in computational complexity to learn the optimal anchor affinity matrix to construct the corresponding graph ($P$ is the number of the anchors in the graph and $P \ll N$). As a result, the storage space and the computational complexity for learning the graph could be decreased.

\captionsetup{labelfont=it, textfont={bf,it}}
\begin{figure*}[t]
\centering
\includegraphics[width=0.8\textwidth]{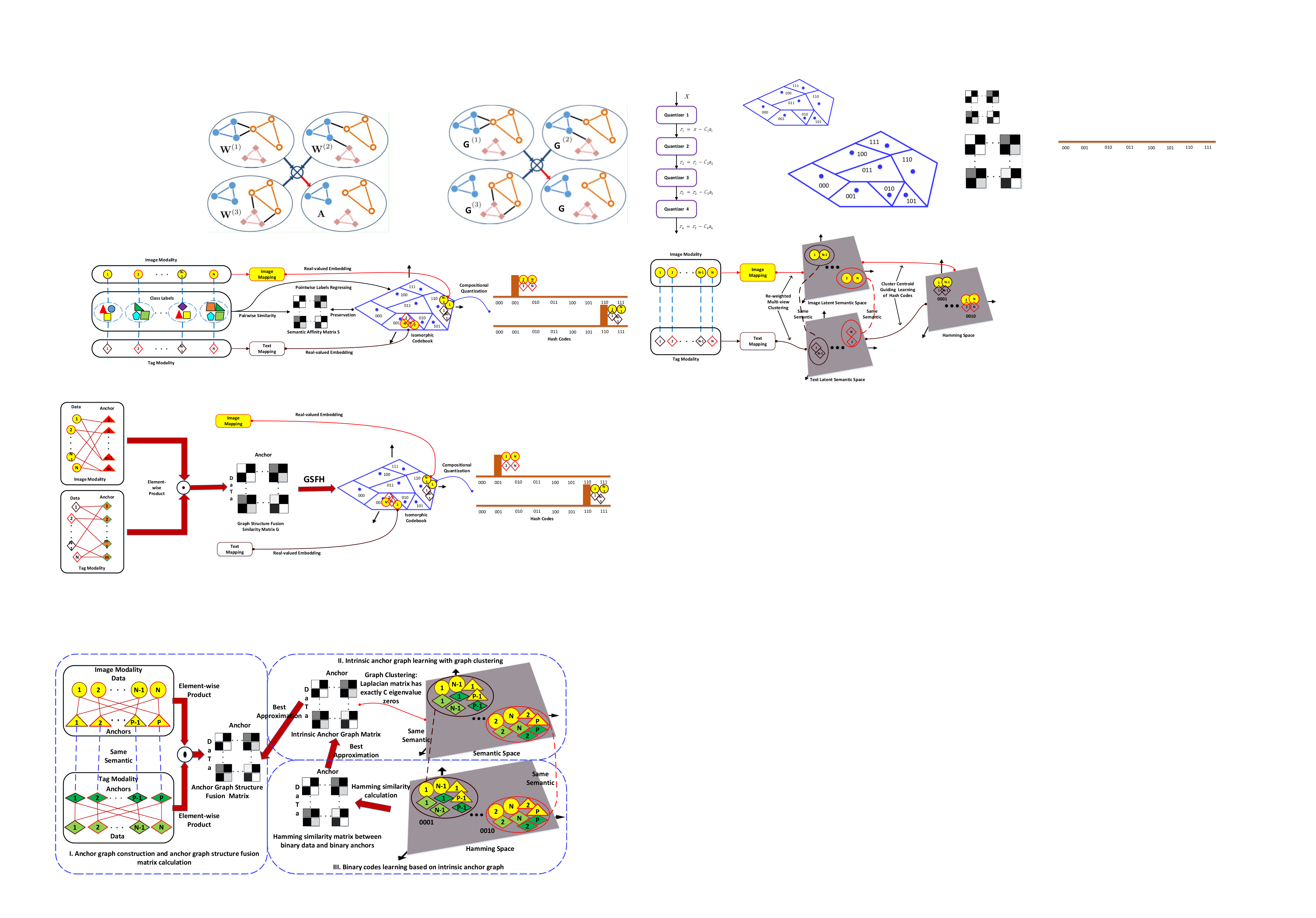}
\caption{The flowchart of AGFSH. We utilize bi-modal data as an example. Firstly, in part I, we construct anchor graphs for image and text modalities and calculate the anchor graph structure fusion matrix by the Hadamard product. Secondly, part II shows intrinsic anchor graph learning with graph clustering by best approximating the anchor graph structure fusion matrix. Based on this process, training instances can be clustered into semantic space. Thirdly, in part III, binary code learning based on intrinsic anchor graph is realized by best approximating Hamming similarity matrix into intrinsic anchor graph. This process can guarantee that binary data could preserve the semantic relationship of training instances in semantic space.}
\label{figure1}
\end{figure*}

To address the above limitations, in this paper, we propose Anchor Graph Structure Fusion Hashing (AGSFH) method to solve effective and efficient large-scale information retrieval across modalities. The flowchart of the proposed model is shown in Fig. \ref{figure1}. By constructing the anchor graph structure fusion matrix from different anchor graphs of multiple modalities with the Hadamard product, AGSFH can fully exploit the geometric property of underlying data structure, leading to the robustness to noise in multi-modal relevance among data instances. The key idea of AGSFH is attempting to directly preserve the fusion anchor affinity with complementary information among multi-modalities data into the common binary Hamming space. At the same time, the structure of the intrinsic anchor graph is adaptively tuned by a well-designed objective function so that the number of components of the intrinsic graph is exactly equal to the number of clusters. Based on this process, training instances can be clustered into semantic space, and binary data could preserve the semantic relationship of training instances. Besides, a discrete optimization framework is designed to learn the unified binary codes across modalities. Extensive experimental results on three public social datasets demonstrate the superiority of AGSFH in cross-modal retrieval. We highlight the main contributions of AGSFH as below.

\begin{enumerate}[1.]
\item We develop an anchor graph structure fusion affinity for large-scale data to directly preserve the anchor fusion affinity with complementary information among multi-modalities data into the common binary Hamming space. The fusion affinity can result in robustness to noise in multi-modal relevance and excellent performance improvement on cross-modal information retrieval.
\item The structure of the intrinsic anchor graph is learned by preserving similarity across modalities and adaptively tuning clustering structure in a unified framework. The learned intrinsic anchor graph can fully exploit the geometric property of underlying data structure across multiple modalities and directly cluster training instances in semantic space.
\item An alternative algorithm is designed to solve the optimization problem in AGSFH. Based on the algorithm, the binary codes are learned without relaxation, avoiding the large quantization error.
\item Extensive experimental results on three public social datasets demonstrate the superiority of AGSFH in cross-modal retrieval.
\end{enumerate}

\section{Related works} \label{Sec2}
Cross-modal hashing has obtained much attention due to its utility and efficiency. Current CMH methods are mainly divided into two categories: unsupervised ones and supervised ones.

Unsupervised CMH methods \cite{proceeding17, proceeding14, proceeding15, proceeding16, jour5, jour23} mainly map the heterogeneous multi-modal data instances into the common compact binary codes by preserving the intra- and inter-modal relevance of training data and learn the hash functions during this process. Fusion Similarity Hashing (FSH) \cite{proceeding17} learns binary hash codes by preserving the fusion similarity of multiple modalities in an undirected asymmetric graph. Collaborative Subspace Graph Hashing (CSGH) \cite{proceeding14} constructs the unified hash codes by a two-stage collaborative learning framework. Joint Coupled-Hashing Representation (JCHR) \cite{proceeding15} learns the unified hash codes via embedding heterogeneous data into their corresponding binary spaces. Hypergraph-based Discrete Hashing (BGDH) \cite{proceeding16} obtains the binary codes with learning hypergraph and binary codes simultaneously. Robust and Flexible Discrete Hashing (RFDH) \cite{jour5} directly produces the hash codes via discrete matrix decomposition. Joint and individual matrix factorization hashing (JIMFH) \cite{jour23} learns unified hash codes with joint matrix factorization and individual hash codes with individual matrix factorization.

Different from unsupervised ones, CVH \cite{proceeding18}, SCM \cite{proceeding19}, SePH \cite{proceeding20}, FDCH \cite{proceeding21}, ADCH \cite{proceeding22} are representative supervised cross-modal hashing methods. Cross View Hashing (CVH) extends the single-modal spectral hashing to multiple modalities and relaxes the minimization problem for learning the hash codes \cite{proceeding18}. Semantic Correlation Maximization (SCM) learns the hash functions by approximating the semantic affinity of label information in large-scale data \cite{proceeding19}. Semantics-Preserving Hashing (SePH) builds a probability distribution by the semantic similarity of data and minimizes the Kullback-Leibler divergence for getting the hash binary codes \cite{proceeding20}. Fast Discrete Cross-modal Hashing (FDCH) learns the hash codes and hash functions with regressing from class labels to binary codes \cite{proceeding21}. Asymmetric Discrete Cross-modal Hashing (ADCH) obtains the common latent representations across the modalities by the collective matrix factorization technique to learn the hash codes and constructs hash functions by a series of binary classifiers \cite{proceeding22}.

Besides, some deep CMH models are developed using deep end-to-end architecture because of the superiority of extracting semantic information in data points. Some representative works are Deep Cross-Modal Hashing (DCMH) \cite{proceeding23}, Dual deep neural networks cross-modal hashing (DDCMH) \cite{proceeding24}, Deep Binary Reconstruction (DBRC) \cite{jour14}. Deep Cross-Modal Hashing (DCMH) learns features and hash codes in the same framework with deep neural networks, one for each modality, to perform feature learning from scratch \cite{proceeding23}. Dual deep neural networks cross-modal hashing (DDCMH) generates hash codes for different modalities by two deep networks, making full use of inter-modal information to obtain high-quality binary codes \cite{proceeding24}. Deep Binary Reconstruction (DBRC) simultaneously learns the correlation across modalities and the binary hashing codes, which also proposes both linear and nonlinear scaling methods to generate efficient codes after training the network \cite{jour14}. In comparison with the traditional hashing methods, deep cross-modal models have shown outstanding performance.

\section{The proposed method}
In this section, we will give a detailed introduction to our proposed AGSFH.

\subsection{Notations and Definitions}
Our AGSFH work belongs to unsupervised hashing approach, achieving state-of-the-art performance in cross-modal semantic similarity search. There are $N$ instances $O = \left\{o_1, o_2, \ldots, o_N\right\}$ in the database set and each instance $o_i = (x_i^1, x_i^2, \ldots, x_i^M)$ has $M$ feature vectors from $M$ modalities respectively. The database matrix $X^m=\left[x_1^m, x_2^m, \ldots , x_N^m \right]\in R^{d_m \times N}$ denotes the feature representations for the $m$th modality, and the feature vector $x_i^m$ is the $i$th data of $X^m$ with $d_m$ dimension. Besides, we assume that the training data set $O$ has $C$ clusters, i.e., $O$ has $C$ semantic categories.

Given training data set $O$, the proposed AGSFH aims to learn a set of hash functions $H^m(x^m)=\left\{h_1^m(x^m), h_2^m(x^m), \ldots, h_K^m(x^m)\right\}$ for the $m$th modal data. At the same time, a common binary code matrix $B=\left[b_1, b_2, \ldots , b_N \right]\in \left\{-1, 1\right\}^{K \times N}$ is constructed, where binary vector $b_i \in \left\{-1, 1\right\}^{K}$ is the $K$-bits code for instance $o_i$. For the $m$th modality data, its hash function can be written as:
\begin{eqnarray}\label{GSFH1}
h_k^m(x^m)=sgn(f_k^m(x^m)), (k=1, 2, \ldots, K),
\end{eqnarray}
where $sgn(\cdot)$ is the sign function, which return $1$ if $f_k^m(\cdot) > 0$ and $-1$ otherwise. $f_k^m(\cdot)$ is the linear or non-linear mapping function for data of the $m$th modality. For simplicity, we define our hash function at the $m$th modality as $H^m(x^m)=sgn(W_m^Tx^m)$.

\subsection{Anchor Graph Structure Fusion Hashing}
In AGSFH, we first construct anchor graphs for multiple modalities and calculate the anchor graph structure fusion matrix by the Hadamard product. Next, AGSFH jointly learns the intrinsic anchor graph and preserves the anchor fusion affinity into the common binary Hamming space. In intrinsic anchor graph learning, the structure of the intrinsic anchor graph is adaptively tuned by a well-designed objective function so that the number of components of the intrinsic graph is exactly equal to the number of clusters. Based on this process, training instances can be clustered into semantic space. Binary code learning based on the intrinsic anchor graph can guarantee that binary data could preserve the semantic relationship of training instances in semantic space.

\subsubsection{Anchor Graph Learning}
Inspired by the idea of GSF \cite{jour15}, it is straightforward to consider our anchor graph structure fusion similarity as follows. However, in general, building a $k$-nearest-neighbor ($k$-NN) graph using all $N$ points from the database always needs $O(N^2)$ in computational complexity. Furthermore, learning an intrinsic graph among all $N$ points also takes $O(N^2)$ in computational complexity. Hence, with the increasing amount of data, it is intractable for off-line learning. To address these computationally intensive problems, inspired by Anchor Graph Hashing (AGH) \cite{proceeding1}, with multi-modal data, different $k$-NN anchor graph affinity matrices $\hat{A}^{m} (m=1,2, \ldots, M)$ are constructed for $M$ modalities respectively. Besides, We further propose an intrinsic anchor graph learning strategy, which learns the intrinsic anchor graph $\hat{S}$ of the intrinsic graph $S$.

In particular, given an anchor set $T=\left\{t_1, t_2, \ldots, t_P\right\}$, where $t_i = (t_i^1, t_i^2, \ldots, t_i^M)$ is the $i$-th anchor across $M$ modalities, which are randomly sampled from the original data set $O$ or are obtained by performing clustering algorithm over $O$. The anchors matrix $T^m=\left[t_1^m, t_2^m, \ldots , t_P^m \right]\in R^{d_m \times P}$ denotes the feature representations of anchors for the $m$th modality, and the feature vector $t_i^m$ is the $i$th anchor of $T^m$. Then, the anchor graph $\hat{A}^{mT}=[\hat{a}_1^m, \hat{a}_2^m, \ldots, \hat{a}_N^m] \in R^{P \times N}$ between all $N$ data points and $P$ anchors at the $m$th modality can be computed as follows. Using the raw data points from $T$ and $O$ respectively represents a pairwise distance $b_{ji}=\left \| x_i^m - t_j^m\right \|_F^2$, we construct the anchor graph $\hat{A}^{mT}$ by using similar initial graph learning of GSF \cite{jour15}. Therefore, we can assign $k$ neighbor anchors to each data instance by the following Eq.(\ref{GSFH14}).
\begin{equation}\label{GSFH14}
\hat{a}_{ij}^{m\star}=
\begin{cases}
\frac{b_{i,k+1}-b_{i,j}}{kb_{i,k+1}-\sum_{j=1}^k b_{i,j}}, & j \leq k \\
0, & otherwise.
\end{cases}
\end{equation}

Since $P \ll N$, the size of the anchor graph is much smaller than the traditional graph, the storage space and the computational complexity for constructing the graph could be decreased. We can obtain different anchor graph affinity matrices $\hat{A}^{m} (m=1, 2, \ldots, M)$ for $M$ modalities respectively.

Similar to \cite{jour15}, we use the Hadamard product to extract intrinsic edges in multiple anchor graphs for fusing different anchor graph structures $\hat{A}^{m}$ into one anchor affinity graph $\hat{A}$ by
\begin{eqnarray}\label{GSFH15}
\hat{A} = \prod_{m=1}^{M} \hat{A}^{m}.
\end{eqnarray}
which can reduce the storage space and the computational complexity of graph structure fusion. Given a fused anchor affinity matrix $\hat{A}$, we learn a anchor similarity matrix $\hat{S}$ so that the corresponding graph $S \simeq \bar{S} = \hat{S}\hat{S}^T $ has exactly $C$ connected components and vertices in each connected component of the graph can be categorized into one cluster.

For the similarity matrix $S \simeq \bar{S}=\hat{S}\hat{S}^T \geq 0$, there is a theorem \cite{book2} about its Laplacian matrix $L$ \cite{jour17}:
\newtheorem{thm}{\bf Theorem}
\begin{thm}\label{thm1}
The number $C$ of connected components of the graph $S$ is equal to the multiplicity of zeros as an eigenvalue of its Laplacian matrix $L$.
\end{thm}
The proof of Fan’s Theorem can be referred to \cite{jour18, jour19}. As we all known, $L$ is a semi-definite matrix. Hence, $L$ has $N$ non-negative eigenvalues $0 = \lambda_1 \leq \lambda_2 \leq \ldots \leq \lambda_N$. Theorem \ref{thm1} tells that if the constraint $\sum_{c=1}^C \lambda_c =0$ is satisfied, the graph $S$ has an ideal neighbors assignment and the data instances are already clustered into $C$ clusters. According to Fan’s theorem \cite{jour21}, we can obtain an objective function,
\begin{eqnarray}\label{GSFH3}
&& \sum_{c=1}^C \lambda_c = \min_{U} \left \langle UU^T, L \right \rangle \nonumber \\
\mathrm{s.t.} && U \in R^{N \times C}, U^TU=I_C,
\end{eqnarray}
where $\left \langle \cdot \right \rangle$ denotes the Frobenius inner product of two matrices, $U^T = [u_1, u_2, \ldots, u_N]$, $L=D-\hat{S}\hat{S}^T$ is the Laplacian matrix, $I_C \in R^{C \times C}$ is an identity matrix, $D$ is a diagonal matrix and its elements are column sums of $\hat{S}\hat{S}^T$.

Furthermore, for intrinsic anchor graph, $\bar{S}$ can be normalized as $\tilde{S} = \hat{S}\Lambda\hat{S}^T$ where $\Lambda= diag(\hat{S}^T1) \in R^{P \times P}$. The approximate intrinsic graph matrix $\tilde{S}$ has a key property as it has unit row and column sums. Hence, the graph laplacian of the intrinsic anchor graph is $L=I-\tilde{S}$, so the required $C$ graph laplacian eigenvectors $U$ in Eq.(\ref{GSFH3}) are also eigenvectors of $\tilde{S}$ but associated with the eigenvalue $1$ (the eigenvalue $1$ corresponding to eigenvalue $0$ of $L$). One can easily find $\tilde{S}=\hat{S}\Lambda\hat{S}^T$ has the same non-zeros eigenvalue with $E=\Lambda^{-1/2}\hat{S}^T \hat{S}\Lambda^{-1/2}$, resulting in $L=I-\tilde{S}$ has the same multiplicity of eigenvalue $0$ with $\hat{L}=I-E$. Hence, similar to Eq.(\ref{GSFH3}), then we have an objective function for anchor graph learning as follows.
\begin{eqnarray}\label{GSFH37}
&& \sum_{c=1}^C \lambda_c = \min_{V} \left \langle VV^T, \hat{L} \right \rangle \nonumber \\
\mathrm{s.t.} && V \in R^{p \times C}, V^TV=I_C,
\end{eqnarray}
where $\hat{S}^T = [\hat{s}_1, \hat{s}_2, \ldots, \hat{s}_N]$ and $V^T = [v_1, v_2, \ldots, v_P]$.

Because the graph $\hat{A}$ contains edges of the intrinsic structure, we need $\hat{S}$ to best approximate $\hat{A}$, then we optimize the following objective function,
\begin{eqnarray}\label{GSFH4}
& & \max_{\hat{S}}  \left \langle \hat{A}, \hat{S} \right \rangle \nonumber \\
\mathrm{s.t.} & & \forall j, \hat{s}_j \geq 0, 1^T\hat{s}_j=1,
\end{eqnarray}
where we constrain $1^T\hat{s}_j=1$ so that it has unit row sum.

By combining Eq. (\ref{GSFH37}) with Eq. (\ref{GSFH4}), we have,
\begin{eqnarray}\label{GSFH6}
& & \min_{V,\hat{S}}  \left \langle VV^T, \hat{L} \right \rangle - \gamma_1 \left \langle \hat{A}, \hat{S} \right \rangle + \gamma_2 \left \| \hat{S} \right \|_F^2 \nonumber \\
\mathrm{s.t.} & & V \in R^{p \times C}, V^TV=I_C, \nonumber \\
              & & \forall j, \hat{s}_j \geq 0, 1^T\hat{s}_j=1,
\end{eqnarray}
where $\gamma_1$ is the weight controller parameter and $\gamma_2$ is the regularization parameter. To avoid trivial solution when optimizing objective with respect to $\hat{s}_j$ in Eq.(\ref{GSFH6}), we add $L_2$-norm regularization to smooth the elements of $\hat{S}$.

We tune the structure of $\hat{S}$ adaptively so that we achieve the condition $\sum_{c=1}^C \lambda_c =0$ to obtain a anchor graph $\hat{S}$ with exactly $C$ number of connected components. As opposed to pre-computing affinity graphs, in Eq.(\ref{GSFH6}), the affinity of the adaptive anchor graph $\hat{S}$, i.e., $\hat{s}_{ij}$, is learned by modeling fused anchor graph $\hat{A}$ from multiple modalities. The learning procedures of multiple modalities mutually are beneficial and reciprocal.

\subsubsection{The Proposed AGSFH Scheme}
Ideally, if the instances $o_i$ and $o_j$ are similar, the Hamming distance between their binary codes should be minimal, and vice versa. We do this by maximizing the approximation between the learned intrinsic similarity matrix $\bar{S}$ and the Hamming similarity matrix $H=B^TB$, which can be written as $\max_{B} \left \langle \bar{S}, B^TB \right \rangle=Tr(B\bar{S}B^T)=Tr(B\hat{S}\hat{S}^TB^T)$. Then, we can assume $B_s = sgn(B\hat{S}) \in \left\{-1, +1\right\}^{K \times P}$ to be the binary anchors. To learn the binary codes, the objective
function can be written as:
\begin{eqnarray}\label{GSFH17}
& & \max_{B,B_s} Tr(B\hat{S}B_s^T), \nonumber \\
\mathrm{s.t.} & & B \in \left\{-1, +1\right\}^{K \times N}, B_s \in \left\{-1, +1\right\}^{K \times P}.
\end{eqnarray}

Intuitively, we learn the $m$th modality hash function $H^m(x^m)$ by minimizing the error term between the linear hash function in Eq.(\ref{GSFH1}) by $\left \| B- H^m(x^m) \right \|_F^2$. Such hash function learning can be easily integrated into the overall cross modality similarity persevering, which is rewritten as:
\begin{eqnarray}\label{GSFH18}
& & \min_{B,B_s,W_{m}} -Tr(B\hat{S}B_s^T)+\lambda \sum_{m=1}^M \left \| B- W_{m}^TX^{m} \right \|_F^2, \nonumber \\
\mathrm{s.t.} & & B \in \left\{-1, +1\right\}^{K \times N}, B_s \in \left\{-1, +1\right\}^{K \times P}.
\end{eqnarray}
where $\lambda$ is a tradeoff parameter to control the weights between minimizing the binary quantization and maximizing the approximation.

Therefore, by combining Eq.(\ref{GSFH6}) with Eq.(\ref{GSFH18}), we have the overall objective as follows.
\begin{eqnarray}\label{GSFH19}
& & \min_{V,\hat{S},B,B_s,W_{m}}  Tr(V^T\hat{L}V) - \gamma_1 Tr(\hat{A}^T \hat{S}) + \gamma_2 \left \| \hat{S} \right \|_F^2 \nonumber \\
&& -\gamma_3 Tr(B\hat{S}B_s^T)+\lambda\sum_{m=1}^M \left \| B- W_{m}^TX^{m} \right \|_F^2 \nonumber \\
\mathrm{s.t.} & & V \in R^{P \times C}, V^TV=I_C, \nonumber \\
              & & \forall j, \hat{s}_j \geq 0, 1^T\hat{s}_j=1, \nonumber\\
              & &  B \in \left\{-1, +1\right\}^{K \times N}, B_s \in \left\{-1, +1\right\}^{K \times P},
\end{eqnarray}
where $\gamma_3$ is a weight controller parameter, $\hat{L}=I-\Lambda^{-1/2}\hat{S}^T \hat{S}\Lambda^{-1/2}$, and $\Lambda= diag(\hat{S}^T1)$.

\subsection{Algorithms Design}
Objective Eq.(\ref{GSFH19}) is a mixed binary programming and non-convex with variables $V$, $\hat{S}$, $B$, $B_s$, and $W_{m}$ together. To solve this issue, an alternative optimization framework is developed, where only one variable is optimized with the other variable fixed at each step. The details of the alternative scheme are as follows.
\begin{enumerate}[1.]
\setlength{\listparindent}{2em}

\item $\textbf{$\hat{S}$ step}$.

\par\setlength\parindent{2em} By fixing $V$, $\Lambda$, $B$, $B_s$, and $W_{m}$, optimizing problem (\ref{GSFH19}) becomes:
\begin{eqnarray}\label{GSFH20}
&& \min_{\hat{S}}  Tr(V^T\hat{L}V) - \gamma_1 Tr(\hat{A}^T \hat{S}) + \gamma_2 \left \| \hat{S} \right \|_F^2 \nonumber \\
&& -\gamma_3 Tr(B\hat{S}B_s^T) \nonumber \\
\mathrm{s.t.}  & & \forall j, \hat{s}_j \geq 0, 1^T\hat{s}_j=1. \nonumber\\
\end{eqnarray}

Note that the problem Eq.(\ref{GSFH20}) is independent between different $j$, then we have,
\begin{eqnarray}\label{GSFH21}
&& \min_{\hat{s}_j}  f(\hat{s}_j) = \hat{s}_j^T (\tilde{V}\tilde{V}^T + \gamma_2 I) \hat{s}_j \nonumber \\
&& - (\gamma_1 \hat{a}_j^T + \gamma_3 b_j^T B_s) \hat{s}_j, \nonumber \\
\mathrm{s.t.} & & \hat{s}_j \geq 0, 1^T\hat{s}_j=1, \nonumber\\
\end{eqnarray}
where $\tilde{V}=\Lambda^{-1/2}V$. We can find the constraints in problem (\ref{GSFH21}) is simplex, which indeed can lead to sparse solution $\hat{s}_j$ and have empirical success in various applications (because $\left \| \hat{s}_j \right \|_1 = 1^T\hat{s}_j=1$).

In order to solve the Eq.(\ref{GSFH21}) for large $P$, it is more appropriate to apply first-order methods. In this paper, we use the Nesterov's accelerated projected gradient method to optimize Eq.(\ref{GSFH21}). We will present the details of the optimization as follows. We can easily find that the objective function (\ref{GSFH21}) is convex, the gradient of the objective function (\ref{GSFH21}) is Lipschitz continuous, and the Lipschitz constant is $Lp=2 \left \| \tilde{V}\tilde{V}^T + \gamma_2 I \right \|_2$ (i.e., the largest singular value of $2(\tilde{V}\tilde{V}^T + \gamma_2 I)$). One can see the detailed proofs of these results in Theorem \ref{thm2} and Theorem \ref{thm3}, respectively. According to these results, problem (\ref{GSFH21}) can be efficiently solved by Nesterov's optimal gradient method (OGM) \cite{jour22}.

\begin{thm}\label{thm2}
The objective function $f(\hat{s}_j)$ is convex.
\end{thm}

\begin{proof}

Given any two vector $\hat{s}_j^1, \hat{s}_j^2 \in R^{P \times 1}$ and a positive number $\mu \in (0,1)$, we have
\begin{eqnarray}\label{GSFH37}
&& f(\mu \hat{s}_j^1 + (1-\mu)\hat{s}_j^2) - (\mu f(\hat{s}_j^1) + (1-\mu)f(\hat{s}_j^2)) \nonumber \\
&& = (\mu \hat{s}_j^1 + (1-\mu)\hat{s}_j^2)^T (\tilde{V}\tilde{V}^T + \gamma_2 I) (\mu \hat{s}_j^1 + (1-\mu)\hat{s}_j^2) \nonumber \\
&& - (\gamma_1 \hat{a}_j^T + \gamma_3 b_j^T B_s) (\mu \hat{s}_j^1 + (1-\mu)\hat{s}_j^2) \nonumber \\
&& - \mu (\hat{s}_j^{1T} (\tilde{V}\tilde{V}^T + \gamma_2 I) \hat{s}_j^1 - (\gamma_1 \hat{a}_j^T + \gamma_3 b_j^T B_s) \hat{s}_j^1) \nonumber \\
&& - (1-\mu) (\hat{s}_j^{2T} (\tilde{V}\tilde{V}^T + \gamma_2 I) \hat{s}_j^2 \nonumber \\
 && - (\gamma_1 \hat{a}_j^T + \gamma_3 b_j^T B_s) \hat{s}_j^2)
\end{eqnarray}

By some algebra, (\ref{GSFH37}) is equivalent to
\begin{eqnarray}\label{GSFH38}
&& f(\mu \hat{s}_j^1 + (1-\mu)\hat{s}_j^2) - (\mu f(\hat{s}_j^1) + (1-\mu)f(\hat{s}_j^2)) \nonumber \\
&& = \mu(\mu-1)(\hat{s}_j^1 - \hat{s}_j^2)^T (\tilde{V}\tilde{V}^T + \gamma_2 I)(\hat{s}_j^1 - \hat{s}_j^2) \nonumber \\
&& = \mu(\mu-1)(\left \| \tilde{V}^T (\hat{s}_j^1 - \hat{s}_j^2) \right \|_F^2 + \gamma_2\left \|\hat{s}_j^1 - \hat{s}_j^2\right \|_F^2 ) \nonumber \\
&& \leq 0.
\end{eqnarray}

Therefore, we have
\begin{eqnarray}\label{GSFH39}
&& f(\mu \hat{s}_j^1 + (1-\mu)\hat{s}_j^2) \leq \mu f(\hat{s}_j^1) + (1-\mu)f(\hat{s}_j^2).
\end{eqnarray}

According to the definition of convex function, we know $f(\hat{s}_j)$ is convex. This completes the proof.
\end{proof}

\begin{thm}\label{thm3}
The gradient of the objective function $f(\hat{s}_j)$ is Lipschitz continuous and the Lipschitz constant is $Lp=2 \left \| \tilde{V}\tilde{V}^T + \gamma_2 I \right \|_2$ (i.e., the largest singular value of $2(\tilde{V}\tilde{V}^T + \gamma_2 I)$).
\end{thm}

\begin{proof}

According to (\ref{GSFH21}), we can obtain the gradient of $f(\hat{s}_j)$
\begin{eqnarray}\label{GSFH40}
&& \nabla f(\hat{s}_j) = 2(\tilde{V}\tilde{V}^T + \gamma_2 I) \hat{s}_j - (\gamma_1 \hat{a}_j^T + \gamma_3 b_j^T B_s).
\end{eqnarray}

For any two vector $\hat{s}_j^1, \hat{s}_j^2 \in R^{P \times 1}$, we have
\begin{eqnarray}\label{GSFH41}
&& \left \| \nabla f(\hat{s}_j^1) - \nabla f(\hat{s}_j^2)  \right \|_F^2  \nonumber \\
&& = \left \| 2(\tilde{V}\tilde{V}^T + \gamma_2 I)(\hat{s}_j^1 - \hat{s}_j^2)  \right \|_F^2 \nonumber \\
&& = Tr((U\Sigma U^T(\hat{s}_j^1 - \hat{s}_j^2))^T(U\Sigma U^T(\hat{s}_j^1 - \hat{s}_j^2))),
\end{eqnarray}
where $U\Sigma U^T$ is the SVD decomposition of $2(\tilde{V}\tilde{V}^T + \gamma_2 I)$ and the singular values are $\left \{ \sigma_1, \ldots, \sigma_u \right \}$ listed in a descending order. By some algebra, (\ref{GSFH41}) is equivalent to
\begin{eqnarray}\label{GSFH42}
&& \left \| \nabla f(\hat{s}_j^1) - \nabla f(\hat{s}_j^2)  \right \|_F^2  \nonumber \\
&& = Tr(U^T(\hat{s}_j^1 - \hat{s}_j^2)(\hat{s}_j^1 - \hat{s}_j^2)^T U\Sigma^2) \nonumber \\
&& \leq \sigma_1^2 Tr(U^T(\hat{s}_j^1 - \hat{s}_j^2)(\hat{s}_j^1 - \hat{s}_j^2)^T U) \nonumber \\
&& = \sigma_1^2 \left \| \hat{s}_j^1 - \hat{s}_j^2 \right \|_F^2,
\end{eqnarray}
where $\sigma_1$ is the largest singular value, and the last two equations
come from the fact that $U^TU=I_u$ and $UU^T=I_P$. From (\ref{GSFH42}), we have
\begin{eqnarray}\label{GSFH43}
&& \left \| \nabla f(\hat{s}_j^1) - \nabla f(\hat{s}_j^2)  \right \|_F \nonumber \\
&& \leq  Lp \left \| \hat{s}_j^1 - \hat{s}_j^2 \right \|_F
\end{eqnarray}

Therefore, $\nabla f(\hat{s}_j)$ is Lipschitz continuous and the Lipschitz
constant is the largest singular value of $2(\tilde{V}\tilde{V}^T + \gamma_2 I)$, i.e., $Lp=\left \| 2(\tilde{V}\tilde{V}^T + \gamma_2 I) \right \|_2 = 2 \left \| \tilde{V}\tilde{V}^T + \gamma_2 I\right \|_2$. This completes the proof.
\end{proof}

In particular, we construct two sequences, i.e., $\hat{s}_j^t$ and $z_j^t$, and alternatively update them in each iteration round. For the convenience of notations, we use $C$ to represent the associated constraints in Eq.(\ref{GSFH21}). At the iteration $t$, the two sequences are
\begin{eqnarray}\label{GSFH22}
&\hat{s}_j^t = \arg \min_{\hat{s}_j \in C}  \phi (\hat{s}_j, z_j^{t-1}) = f(z_j^{t-1}) \nonumber \\
& + (\hat{s}_j-z_j^{t-1})^T \nabla f(z_j^{t-1}) + \frac{Lp}{2} \left \| \hat{s}_j-z_j^{t-1} \right \|_2^2 ,
\end{eqnarray}
and
\begin{eqnarray}\label{GSFH23}
z_j^t = \hat{s}_j^t + \frac{c_t - 1}{c_{t+1}} (\hat{s}_j^t - \hat{s}_j^{t-1}),
\end{eqnarray}
where $\phi (\hat{s}_j, z_j^{t-1})$ is the proximal function of $f(\hat{s}_j)$ on $z_j^{t-1}$, $\hat{s}_j^t$ includes the approximate solution obtained by minimizing the proximal function over $\hat{s}_j$, and $z_j^t$ stores the search point that is constructed by linearly combining the latest two approximate solutions, i.e., $\hat{s}_j^t$ and $\hat{s}_j^{t-1}$. According to \cite{jour22}, the combination coefficient is updated in each iteration round
\begin{eqnarray}\label{GSFH24}
c_{t+1} = \frac{1+\sqrt{4c_t+1}}{2}.
\end{eqnarray}

With extra terms independent of $\hat{s}_j$, we could write the objective function in Eq.(\ref{GSFH22}) into a more compact form as follows:
\begin{eqnarray}\label{GSFH25}
\hat{s}_j^t = \arg \min_{\hat{s}_j \in C}  \frac{Lp}{2} \left \| \hat{s}_j-(z_j^{t-1}-\frac{1}{Lp} \nabla f(z_j^{t-1})) \right \|_2^2 .
\end{eqnarray}
Eq.(\ref{GSFH25}) is an euclidean projection problem on the simplex space, which is the same as Eq.(\ref{GSFH7}). According to the Karush-Kuhn-Tucker condition \cite{book1}, it can be verified that the optimal solution $\hat{s}_j^{t}$ is
\begin{eqnarray}\label{GSFH26}
\hat{s}_j^{t}=(z_j^{t-1}-\frac{1}{Lp} \nabla f(z_j^{t-1})+\eta1)_+.
\end{eqnarray}

By alternatively updating $\hat{s}_j^t$, $z_j^t$ and $c_{t+1}$ with (\ref{GSFH22}), (\ref{GSFH23}) and (\ref{GSFH24}) until convergence, the optimal solution can be obtained. Note that recent results \cite{jour20, jour22} show that the gradient-based methods with smooth optimization can achieve the optimal convergence rate $O(\frac{1}{t^2})$, here $t$ is number of iterations. Here the convergence criteria is that the relative change of $\left \| \hat{s}_j \right \|_2$ is less than $10^{-4}$. We initialize $\hat{s}_j^0$ via solving Eq.(\ref{GSFH21}) without considering the constraints. We take the partial derivative of the objective (\ref{GSFH21}) with respect to $\hat{s}_j$. By setting this partial derivative to zero, a closed-form solution of $\hat{s}_j^0$ is acquired
\begin{eqnarray}\label{GSFH27}
\hat{s}_j^{0}= (\tilde{V}\tilde{V}^T + \gamma_2 I)^{-1}(\gamma_1 \hat{a}_j + \gamma_3 B_s^T b_j).
\end{eqnarray}

The full OMG algorithm is summarized in Algorithm \ref{alg:OGM}.

\begin{algorithm}[h]
  \caption{Optimal Gradient Method (OGM).}
  \label{alg:OGM}
  \begin{algorithmic}[1]
    \Require
      $\tilde{V}$, $B$, $B_s$, $\hat{A}$, $\gamma_1$, $\gamma_2$, $\gamma_3$; maximum iteration number $T_{OGM}$.
    \Ensure
      $\hat{S}$.
    \State Initialize $j=1$.
    \Repeat
      \State Initialize $\hat{s}_j^0$ by Eq.(\ref{GSFH27}), $z_j^0 = \hat{s}_j^0$, $t=1$, $c_1=1$.
      \Repeat
        \State $\hat{s}_j^{t}=(z_j^{t-1}-\frac{1}{Lp} \nabla f(z_j^{t-1})+\eta1)_+$.
        \State $c_{t+1} = \frac{1+\sqrt{4c_t+1}}{2}$.
        \State $z_j^t = \hat{s}_j^t + \frac{c_t - 1}{c_{t+1}} (\hat{s}_j^t - \hat{s}_j^{t-1})$.
      \Until Convergence criteria is satisfied or reaching the maximum iteration.
      \State $\hat{s}_j=\hat{s}_j^{t}$.
      \State $j=j+1$.
    \Until $j$ is equal to $N$.
    \State $\hat{S}^T = [\hat{s}_1, \hat{s}_2, \ldots, \hat{s}_N]$.
  \end{algorithmic}
\end{algorithm}

\item $\textbf{$\Lambda$ step}$.

By fixing $V$, $\hat{S}$, $B$, $B_s$, and $W_{m}$, it is easy to solve $\Lambda$ by
\begin{eqnarray}\label{GSFH28}
\Lambda= diag(\hat{S}^T1).
\end{eqnarray}

\item $\textbf{$V$ step}$.

By fixing $\hat{S}$, $\Lambda$, $B$, $B_s$, and $W_{m}$, then, Eq.(\ref{GSFH19}) becomes
\begin{eqnarray}\label{GSFH29}
&& \min_{V} Tr(V^T \hat{L} V) \nonumber \\
\mathrm{s.t.} & & V \in R^{P \times C}, V^TV=I_C.
\end{eqnarray}

The optimal $V$ for Eq.(\ref{GSFH29}) is formed by the $C$ eigenvectors corresponding to the top $C$ smallest eigenvalues of the normalization Laplacian matrix $\hat{L}$.

\item $\textbf{$B$ step}$. \label{sub1}

By fixing $V$, $\hat{S}$, $\Lambda$, $B_s$, and $W_{m}$, the corresponding sub-problem is:
\begin{eqnarray}\label{GSFH30}
& & \min_{B} -\gamma_3 Tr(B\hat{S}B_s^T)+\lambda\sum_{m=1}^M \left \| B- W_{m}^TX^{m} \right \|_F^2 \nonumber \\
\mathrm{s.t.} & &  B \in \left\{-1, +1\right\}^{K \times N},
\end{eqnarray}
which can be expended into:
\begin{eqnarray}\label{GSFH31}
& & \min_{B} - Tr(B(\gamma_3 \hat{S}B_s^T + 2 \lambda \sum_{m=1}^M X^{mT}W_{m} ) \nonumber \\
\mathrm{s.t.} & &  B \in \left\{-1, +1\right\}^{K \times N}.
\end{eqnarray}

This sub-problem can be solved by the following updating result:
\begin{eqnarray}\label{GSFH32}
B = sgn(\gamma_3 B_s \hat{S}^T + 2 \lambda \sum_{m=1}^M W_{m}^T X^{m}).
\end{eqnarray}

\item $\textbf{$B_s$ step}$.

By fixing $V$, $\hat{S}$, $\Lambda$, $B$, and $W_{m}$, the updating of $B_s$ can be referred to:
\begin{eqnarray}\label{GSFH33}
& & \min_{B_s} - \gamma_3 Tr(B\hat{S}B_s^T) \nonumber \\
\mathrm{s.t.} & &  B_s \in \left\{-1, +1\right\}^{K \times P}.
\end{eqnarray}

With the same scheme to deal with sup-problem (\ref{sub1}), this subproblem can be solved as follow:
\begin{eqnarray}\label{GSFH34}
B_s = sgn(B\hat{S}).
\end{eqnarray}
which is similar to the before assumption in Eq.(\ref{GSFH17}).

\item $\textbf{$W_{m}$ step}$.

By fixing $V$, $\hat{S}$, $\Lambda$, $B$, and $B_s$, this sub-problem finds the best mapping coefficient $W_{m}$ by minimizing $\left \| B- W_{m}^TX^{m} \right \|_F^2 $ with the traditional linear regression. Therefore, we update $W_{m}$ as:
\begin{eqnarray}\label{GSFH35}
W_{m} = (X^{m}X^{mT})^{-1}X^{m}B^T.
\end{eqnarray}

\end{enumerate}

The full AGSFH algorithm is summarized in Algorithm \ref{alg:GSFH}.

\begin{algorithm}[h]
  \caption{Anchor Graph Structure Fusion Hashing (AGSFH).}
  \label{alg:GSFH}
  \begin{algorithmic}[1]
    \Require
      feature matrices $X^{m} (m=1,2,\ldots,M)$; code length $K$, the number of anchor points $P$,  the cluster number $C$, the number of neighbor points $k$, maximum iteration number $T_{iter}$; parameters $\gamma_1$, $\gamma_2$, $\gamma_3$, $\lambda$.
    \Ensure
      The hash codes $B$ for training instances $O$ and the projection coefficient matrix $W_{m} (m=1,2,\ldots,M)$.
    \State Uniformly and randomly select $P$ sample pairs from training instances as the anchors $T$.
    \State Construct anchor graph $\hat{A}^{m} (m=1,2,\ldots,M)$ from the data matrix $X^{m}$ and the anchor matrix $T^{m}$ by a $k$-NN graph algorithm.
    \State Calculate anchor graph structure fusion matrix $\hat{A}$ by Eq.(\ref{GSFH15}).
    \State Initialize $V$ by $C$ number of eigenvectors corresponding to the top $C$ smallest eigenvalues of the Laplacian matrix $\hat{L}=I-D^{-1/2}\hat{A}^T \hat{A}D^{-1/2}$, and $D= diag(\hat{A}^T1)$.
    \State Initialize $\Lambda = I_P$.
    \State Initialize $W_{m} (m=1,2,\ldots,M)$ randomly.
    \State Initialize hash codes $B$ and $B_s$ randomly, such that $-1$ and $1$ in each bit are balanced.
    \Repeat
      \State Update $\hat{S}$ by Algorithm \ref{alg:OGM}.
      \State Update $\Lambda$ by Eq.(\ref{GSFH28}).
      \State Update $V$ by Eq.(\ref{GSFH29}), i.e., $V$ is formed by $C$ eigenvectors with the top $C$ smallest eigenvalues of $\hat{L}=I-\Lambda^{-1/2}\hat{S}^T \hat{S}\Lambda^{-1/2}$.
      \State Update $B$ by Eq.(\ref{GSFH32}).
      \State Update $B_s$ by Eq.(\ref{GSFH34}).
      \State Update $W_{m} (m=1,2,\ldots,M)$ by Eq.(\ref{GSFH35}).
    \Until Objective function of Eq.(\ref{GSFH19}) converges or reaching the maximum iteration.
  \end{algorithmic}
\end{algorithm}

\subsection{Convergence Analysis}
The original problem Eq.(\ref{GSFH19}) is not a joint convex problem of $\hat{S}$, $V$, $B$, $B_s$, and $W_{m}$. Hence, we may not obtain a global solution. We divide the original problem into six subproblems, i.e., Eqs.(\ref{GSFH21}), (\ref{GSFH28}), (\ref{GSFH29}), (\ref{GSFH30}), (\ref{GSFH33}) and (\ref{GSFH35}). Since Eq.(\ref{GSFH21}) is a constrained quadratic minimization, Eq.(\ref{GSFH28}) is a linear equation, $\hat{L}$ in Eq.(\ref{GSFH29}) is positive semi-definite, Eqs.(\ref{GSFH30}), (\ref{GSFH33}) are constrained linear minimization, and Eq.(\ref{GSFH35}) is a quadratic minimization, each of them is the convex problem. The six subproblems are solved alternatively, so AGSFH will converge to a local solution. In the section \ref{Convergences}, we will show the convergence curves.

\subsection{Computational Analysis}
The complexity of the proposed AGSFH mainly consists of six
parts: 1) updating intrinsic anchor graph $\hat{S}$, 2) updating $\Lambda$, 3) calculating $C$ eigenvectors of $\hat{L}$, 4) updating hash codes $B$, 5) updating anchor hash codes $B_s$, 6) updating projection coefficients $W_{m} (m=1,2,\ldots,M)$. These six parts are repeated until meeting the convergence condition, and they take $O (T_{OGM}(P+CP+KP+P^2+CP^2+P^3)N)$, $O (NP)$, $O (CP^2+P^2N)$, $O (K(P+1+\sum_{m=1}^M d_m)N)$, $O (KPN)$, and $O (\sum_{m=1}^M (d_m^2+Kd_m)N+\sum_{m=1}^M (d_m^2+d_m^3))$ respectively. Thus, the total complexity of Eq.(\ref{GSFH19}) is
\begin{eqnarray}\label{GSFH36}
&O (T_{iter}(T_{OGM}(P+CP+KP+P^2+CP^2+P^3) \nonumber \\
&+K(P+1+\sum_{m=1}^M d_m)+\sum_{m=1}^M (d_m^2+Kd_m) \nonumber \\
&+P+P^2+KP)N),
\end{eqnarray}
where $T_{iter}$ is the total number of iterations. We can find AGSFH has a linear training time complexity to the training samples.

\section{Experiments}
\renewcommand{\arraystretch}{1.0}
\begin{table}[htb]
  \centering
  \footnotesize
  %\fontsize{6.5}{8}\selectfont
  \setlength{\belowcaptionskip}{10pt}
  \caption{Comparison of MAP with Two Cross-modal Retrieval Tasks on Wiki Benchmark.}
  \label {Table.1}
  \resizebox{0.45\textwidth}{!}{
    \begin{tabular}{ccccccc}
    \hline
    \multirow{2}{*}{Tasks}& \multirow{2}{*}{Methods} &
    \multicolumn{4}{c}{Wiki}\cr\cline{3-6}
    &&16 bits&32 bits&64 bits&128 bits\cr
    \hline
    \multirow{6}{*}{I$\rightarrow$T}
    &CSGH &0.2065	&0.2131	&0.1985	&0.1983
\cr
    &BGDH &0.1815	&0.1717	&0.1717	&0.1717
\cr
    &FSH &0.2426	&0.2609	&0.2622	&{\bf 0.2710}
\cr
    &RFDH &0.2443	&0.2455	&0.2595	&0.2616
\cr
    &JIMFH &0.2384	&0.2501	&0.2472	&0.2542
\cr
    &{\bf AGSFH} &{\bf 0.2548}	&{\bf 0.2681}	&{\bf 0.2640}	&0.2680
\cr
    \hline
    \multirow{6}{*}{T$\rightarrow$I}
    &CSGH &0.2130	&0.2389	&0.2357	&0.2380
\cr
    &BGDH &0.1912	&0.1941	&0.2129	&0.2129
\cr
    &FSH &0.4150	&0.4359	&0.4753	&0.4956
\cr
    &RFDH &0.4185	&0.4438	&0.4633	&0.4922
\cr
    &JIMFH &0.3653	&0.4091	&0.4270	&0.4456
\cr
    &{\bf AGSFH} &{\bf 0.5782}	&{\bf 0.6005}	&{\bf 0.6175}	&{\bf 0.6214}
\cr
    \hline
    \end{tabular}}
\end{table}

\renewcommand{\arraystretch}{1.0}
\begin{table}[htb]
  \centering
  \footnotesize
  %\fontsize{6.5}{8}\selectfont
  \setlength{\belowcaptionskip}{10pt}
  \caption{Comparison of MAP with Two Cross-modal Retrieval Tasks on MIRFlickr25K Benchmark.}
  \label {Table.2}
  \resizebox{0.45\textwidth}{!}{
    \begin{tabular}{ccccccc}
    \hline
    \multirow{2}{*}{Tasks}& \multirow{2}{*}{Method} &
    \multicolumn{4}{c}{MIRFlickr25K}\cr\cline{3-6}
    &&16 bits&32 bits&64 bits&128 bits\cr
    \hline
    \multirow{6}{*}{I$\rightarrow$T}
    &CSGH &0.5240	&0.5238	&0.5238	&0.5238
\cr
    &BGDH &0.5244	&0.5248	&0.5248	&0.5244
\cr
    &FSH &0.6347	&0.6609	&0.6630	&0.6708
\cr
    &RFDH &0.6525	&0.6601	&0.6659	&0.6659
\cr
    &JIMFH &{\bf 0.6563}	&{\bf 0.6703}	&0.6737	&0.6813
\cr
    &{\bf AGSFH} &0.6509	&0.6650	&{\bf 0.6777}	&{\bf 0.6828}
\cr
    \hline
    \multirow{6}{*}{T$\rightarrow$I}
    &CSGH &0.5383	&0.5381	&0.5382	&0.5379
\cr
    &BGDH &0.5360	&0.5360	&0.5360	&0.5360
\cr
    &FSH &0.6229	&0.6432	&0.6505	&0.6532
\cr
    &RFDH &0.6389	&0.6405	&0.6417	&0.6438
\cr
    &JIMFH &0.6432	&0.6570	&0.6605	&0.6653
\cr
    &{\bf AGSFH} &{\bf 0.6565}	&{\bf 0.6862}	&{\bf 0.7209}	&{\bf 0.7505}
\cr
    \hline
    \end{tabular}}
\end{table}

\renewcommand{\arraystretch}{1.0}
\begin{table}[htb]
  \centering
  \footnotesize
  \setlength{\belowcaptionskip}{10pt}
  \caption{Comparison of MAP with Two Cross-modal Retrieval Tasks on NUS-WIDE Benchmark.}
  \label {Table.3}
  \resizebox{0.45\textwidth}{!}{
    \begin{tabular}{ccccccc}
    \hline
    \multirow{2}{*}{Tasks}& \multirow{2}{*}{Method} &
    \multicolumn{4}{c}{NUS-WIDE}\cr\cline{3-6}
    &&16 bits&32 bits&64 bits&128 bits\cr
    \hline
    \multirow{6}{*}{I$\rightarrow$T}
    &CSGH &0.4181	&0.4550	&0.4551	&0.4652
\cr
    &BGDH &0.4056	&0.4056	&0.4056	&0.4056
\cr
    &FSH &{\bf 0.5021}	&0.5200	&0.5398	&{\bf 0.5453}
\cr
    &RFDH &0.4701	&0.4699	&0.4611	&0.4772
\cr
    &JIMFH &0.4952	&{\bf 0.5334}	&0.5223	&0.5334
\cr
    &{\bf AGSFH} &0.4856	&0.5189	&{\bf 0.5401}	&0.5433
\cr
    \hline
    \multirow{6}{*}{T$\rightarrow$I}
    &CSGH &0.4505	&0.5132	&0.5201	&0.5121
\cr
    &BGDH &0.3856	&0.3851	&0.3851	&0.3850
\cr
    &FSH &0.4743	&0.4953	&0.5114	&0.5327
\cr
    &RFDH &0.4701	&0.4713	&0.4651	&0.4626
\cr
    &JIMFH &0.4613	&0.4757	&0.5107	&0.5179
\cr
    &{\bf AGSFH} &{\bf 0.5152}	&{\bf 0.5834}	&{\bf 0.6144}	&{\bf 0.6362}
\cr
    \hline
    \end{tabular}}
\end{table}

\begin{figure*} [ht]
\centering
\subfigure[ ]{
  \includegraphics[width=.3\textwidth]{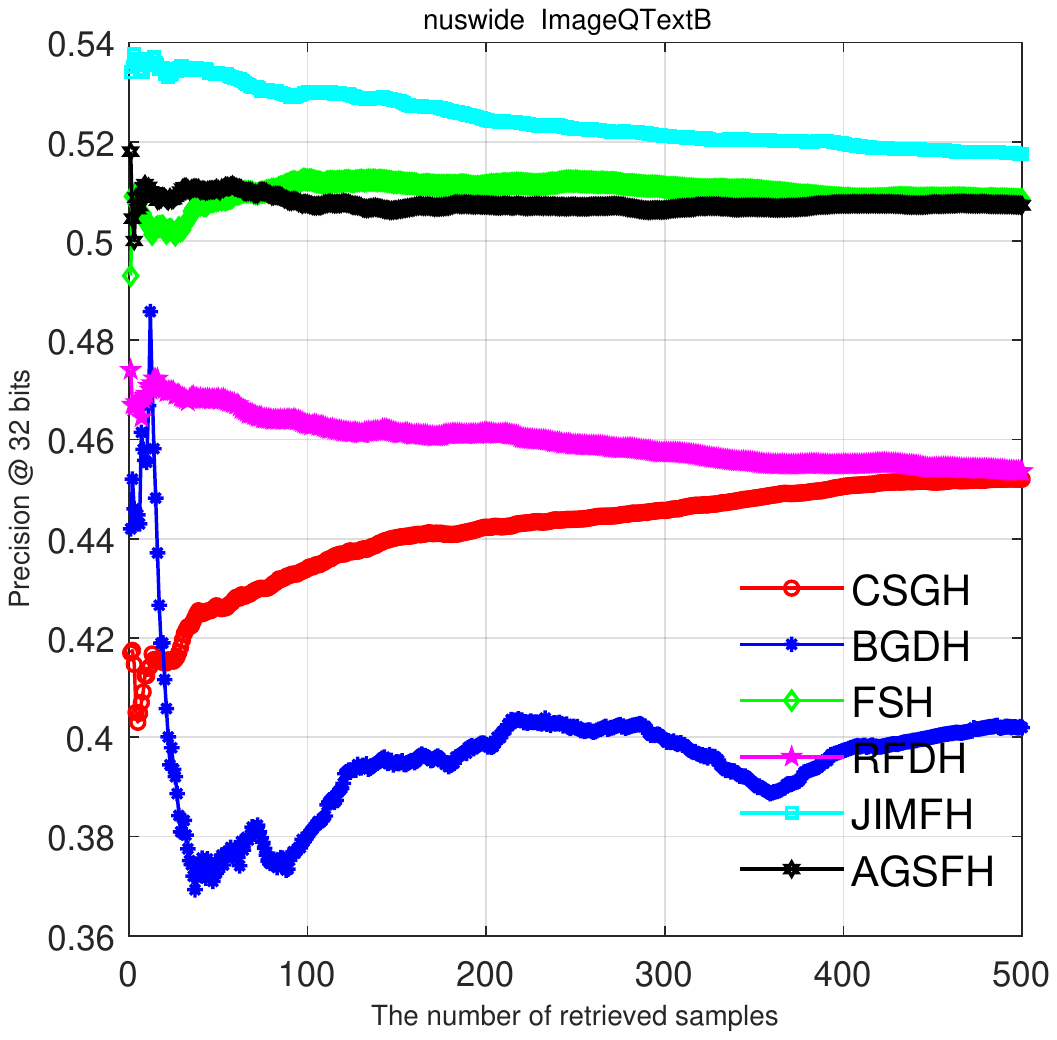}}
\hfill
\centering
\subfigure[ ]{
  \includegraphics[width=.3\textwidth]{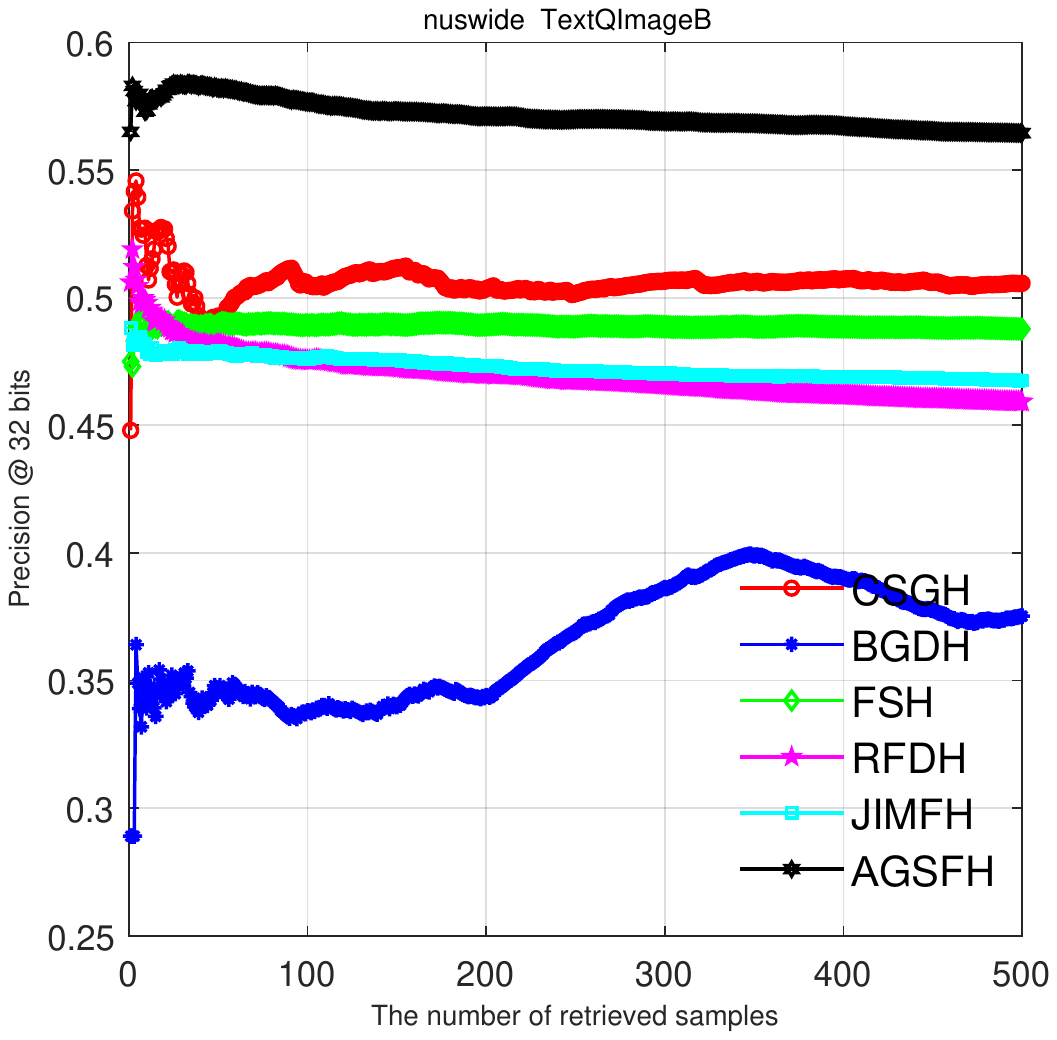}}
\hfill
\centering
\subfigure[ ]{
  \includegraphics[width=.3\textwidth]{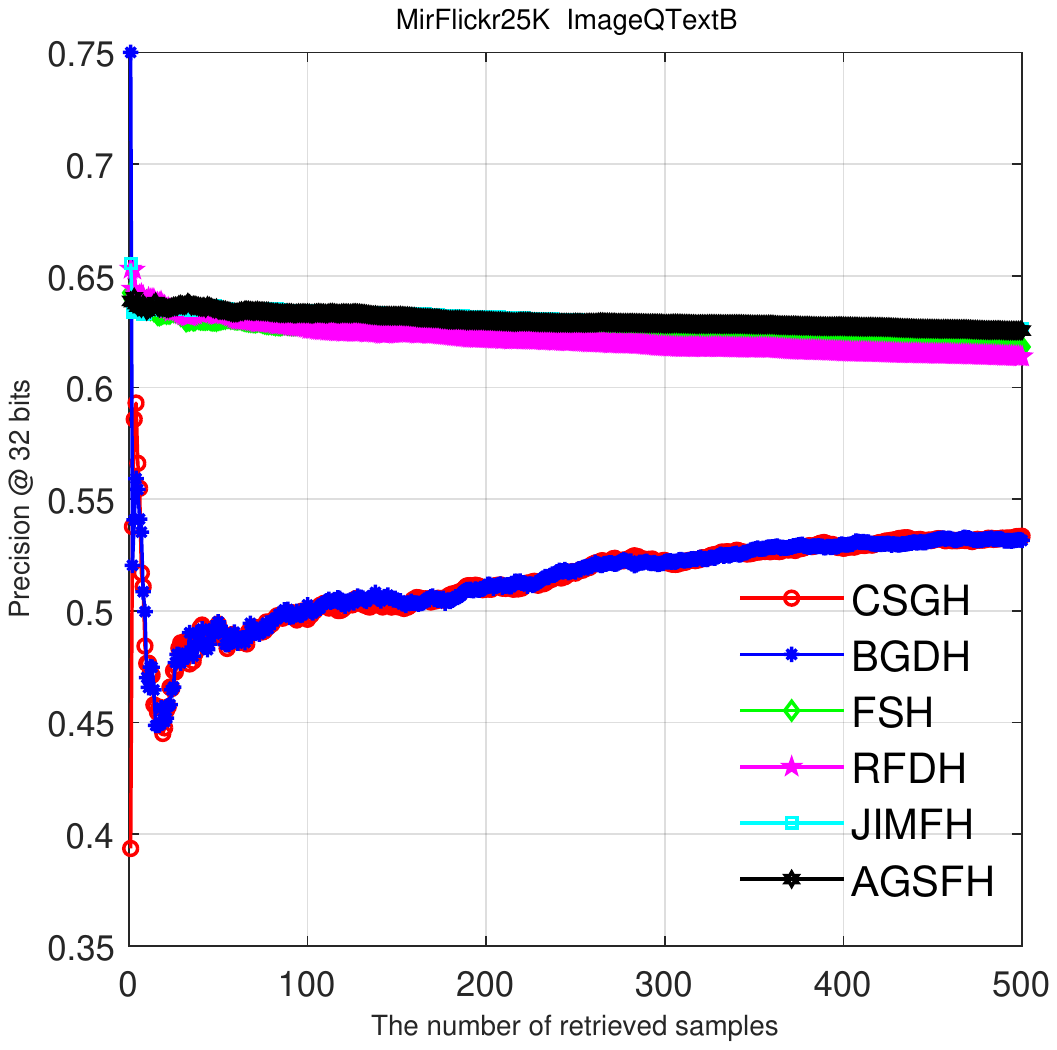}}
\centering
\subfigure[ ]{
  \includegraphics[width=.3\textwidth]{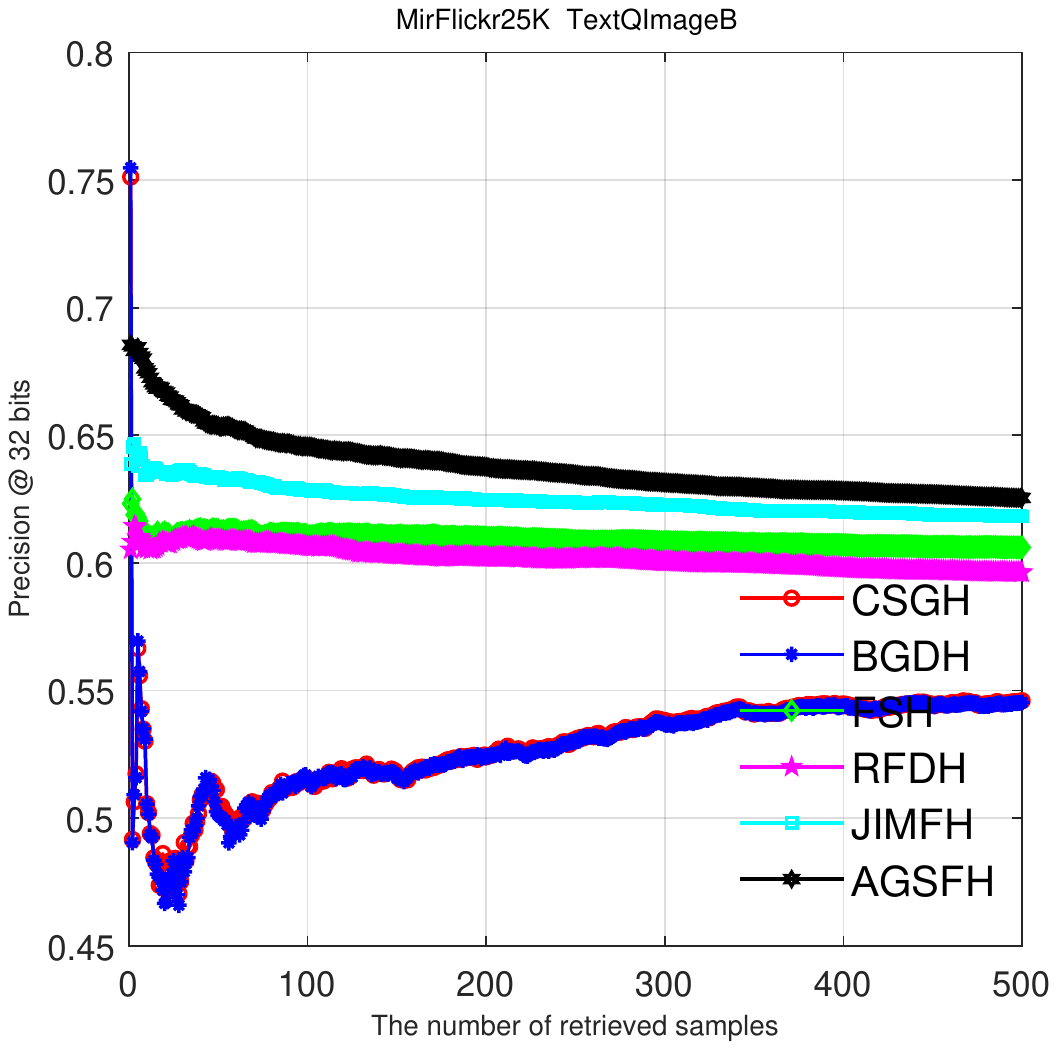}}
\hfill
\centering
\subfigure[ ]{
  \includegraphics[width=.3\textwidth]{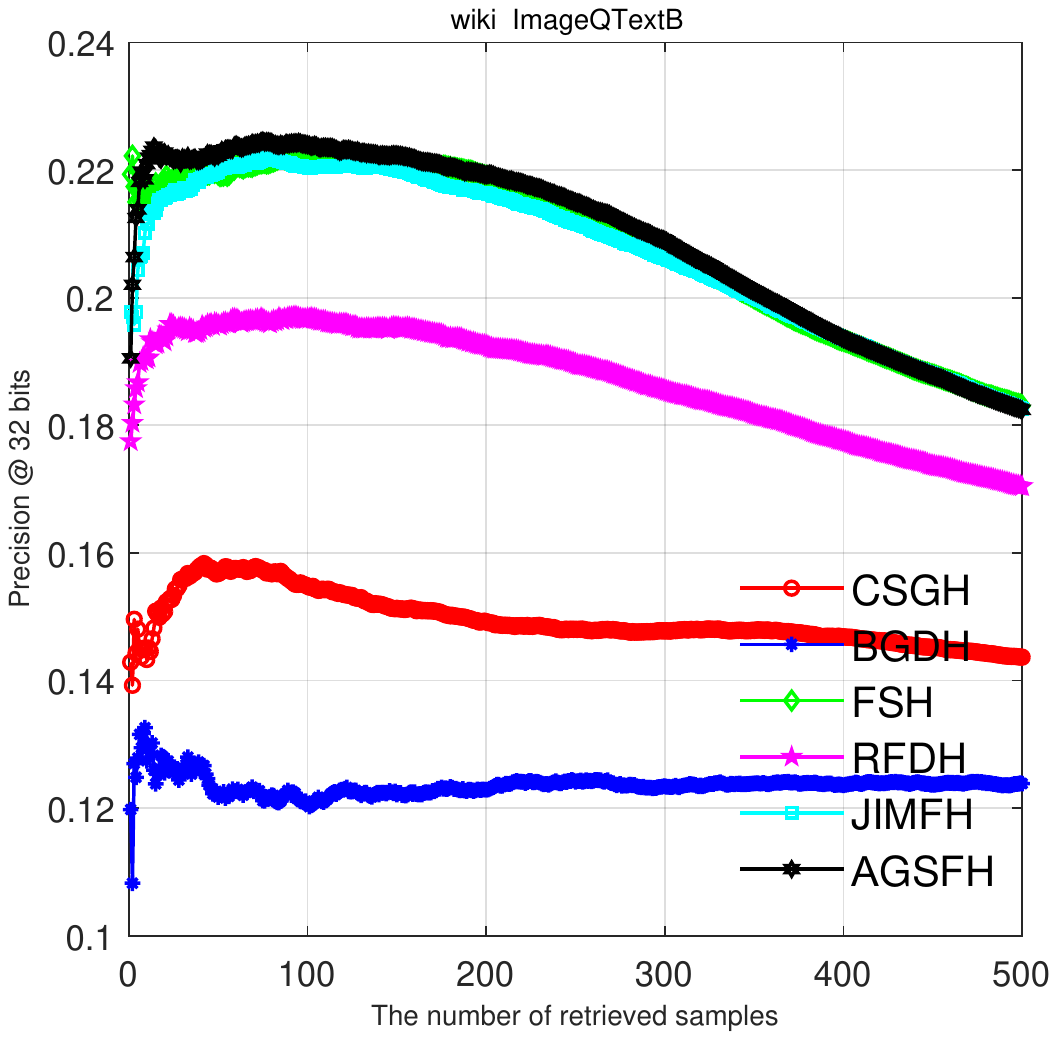}}
\hfill
\centering
\subfigure[ ]{
  \includegraphics[width=.3\textwidth]{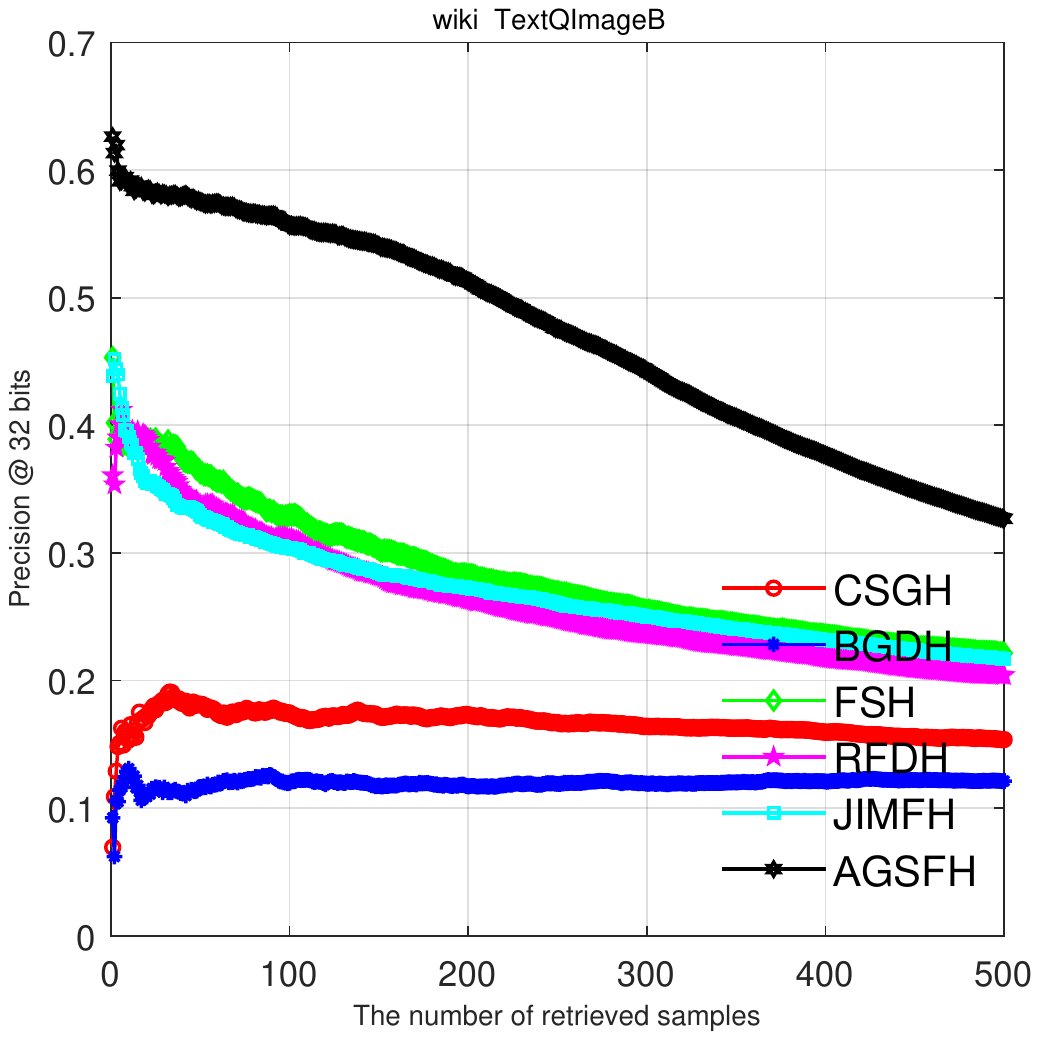}}
\caption{TopN-precision Curves @ $32$ bits on three cross-modal benchmark datasets. (a) and (b) are Nuswide. (c) and (d) are MirFlickr25K. (e) and (f) are Wiki.}
\label{figure3}
\end{figure*}

\begin{figure*} [ht]
\centering
\subfigure[ ]{
  \includegraphics[width=.3\textwidth]{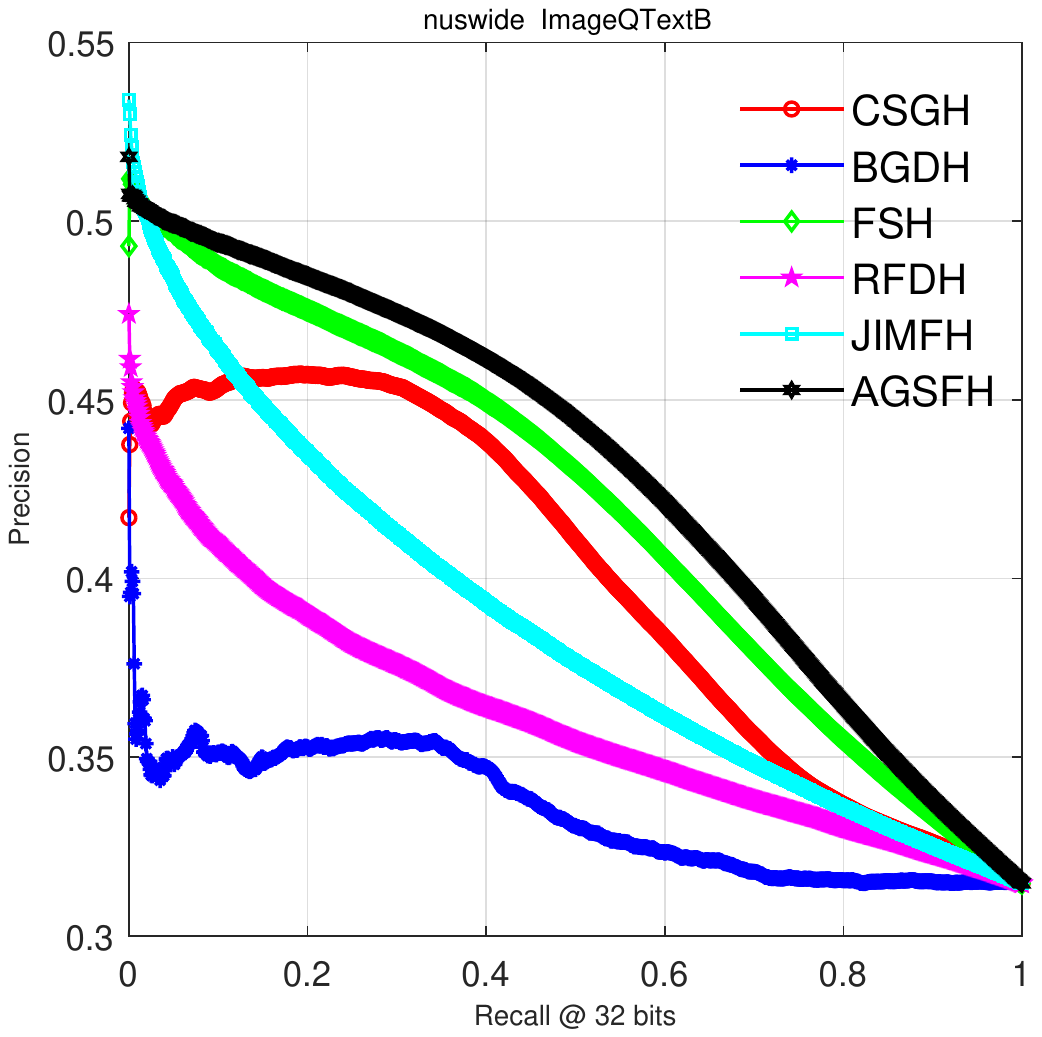}}
\hfill
\centering
\subfigure[ ]{
  \includegraphics[width=.3\textwidth]{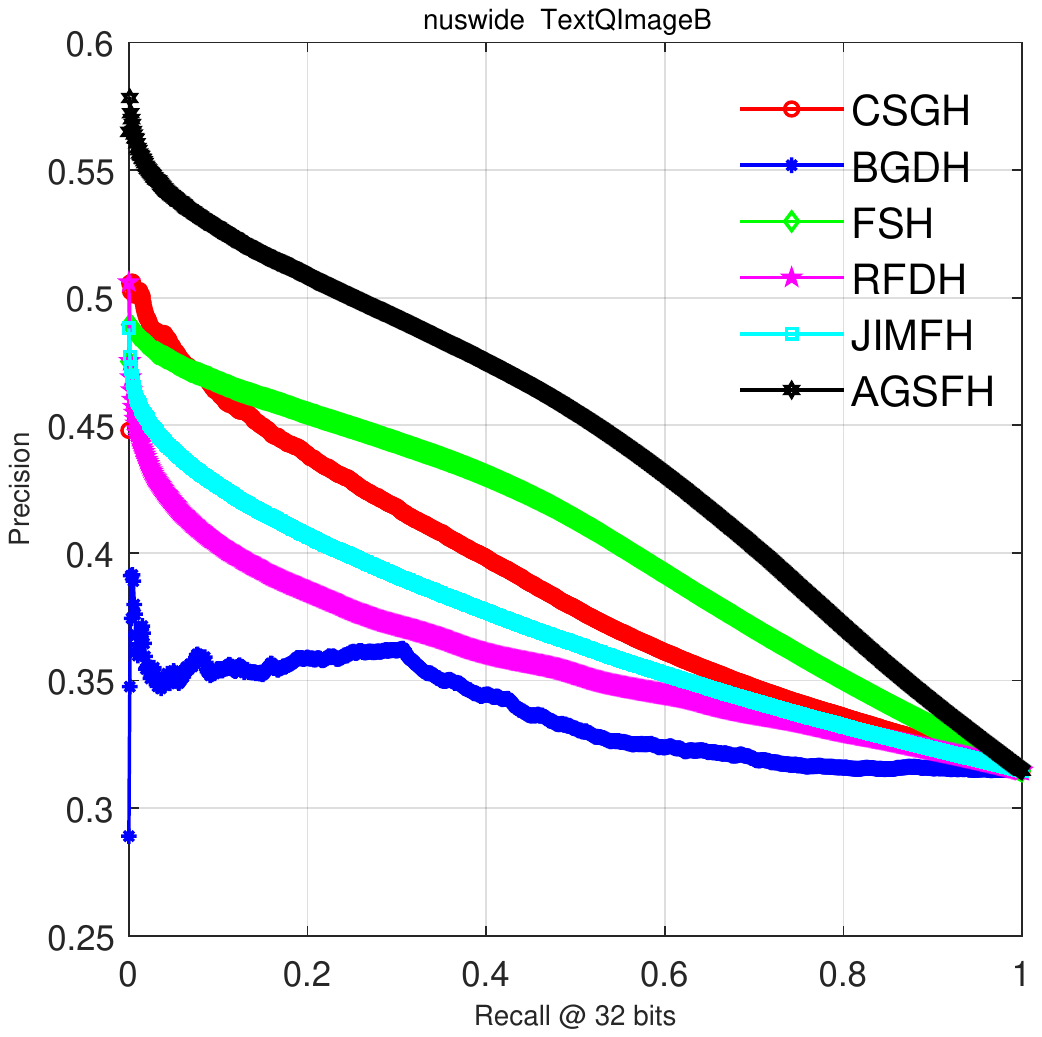}}
\hfill
\centering
\subfigure[ ]{
  \includegraphics[width=.3\textwidth]{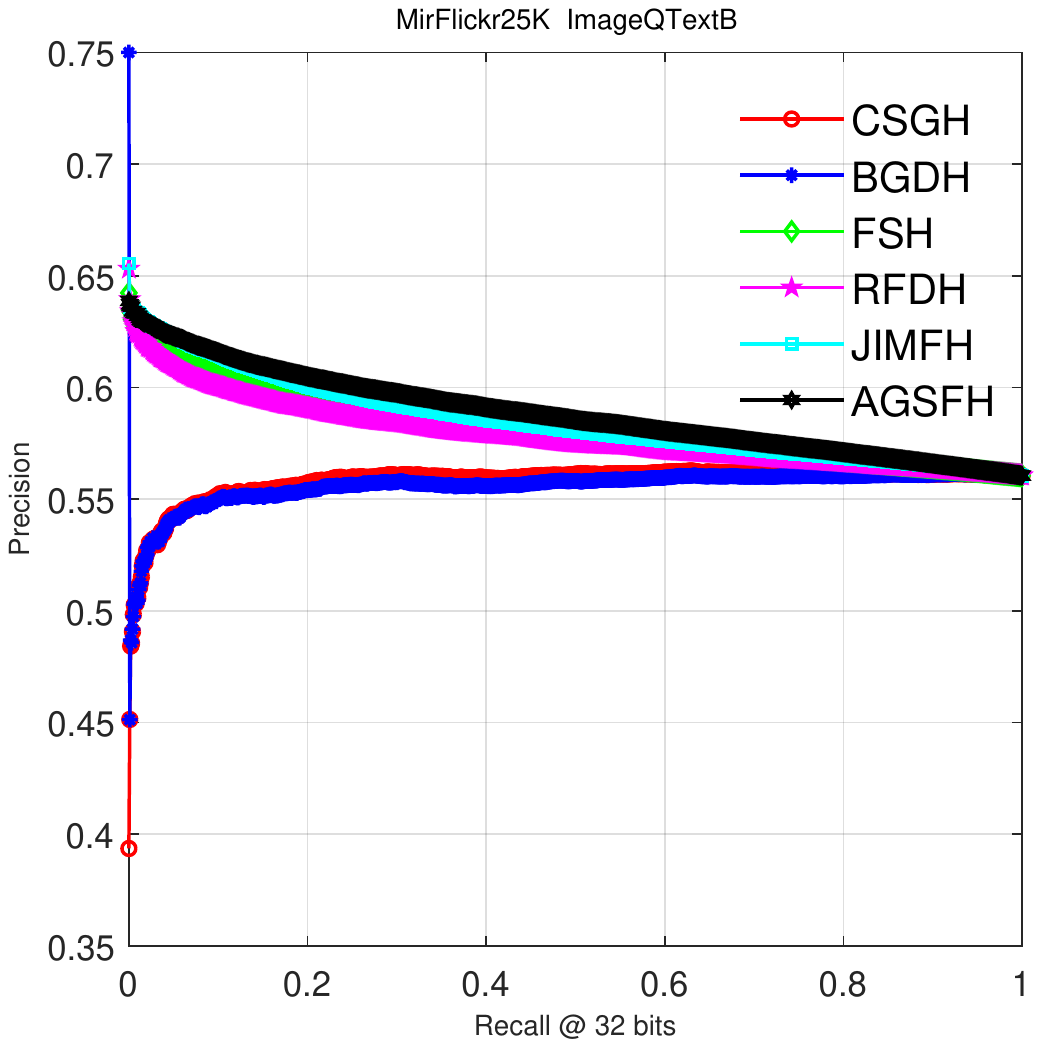}}
\centering
\subfigure[ ]{
  \includegraphics[width=.3\textwidth]{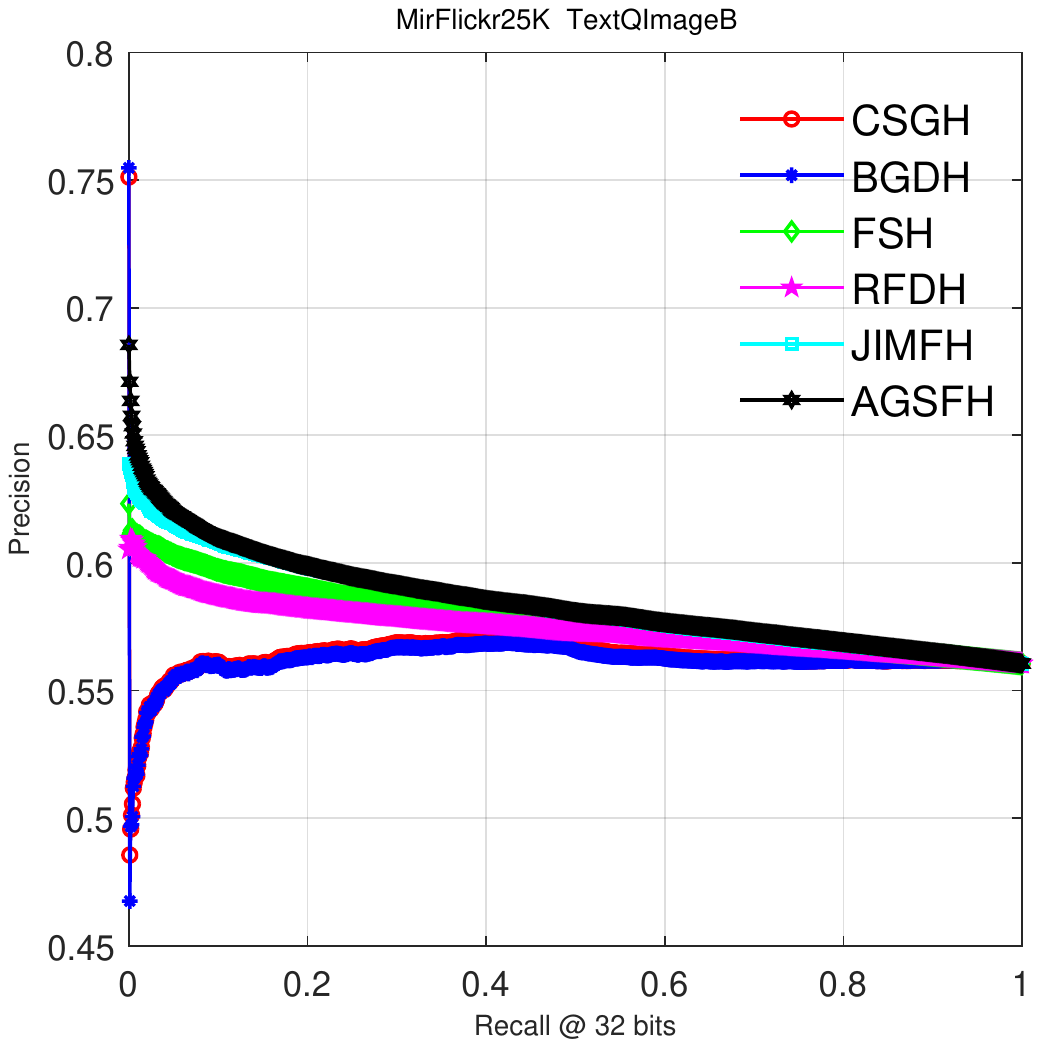}}
\hfill
\centering
\subfigure[ ]{
  \includegraphics[width=.3\textwidth]{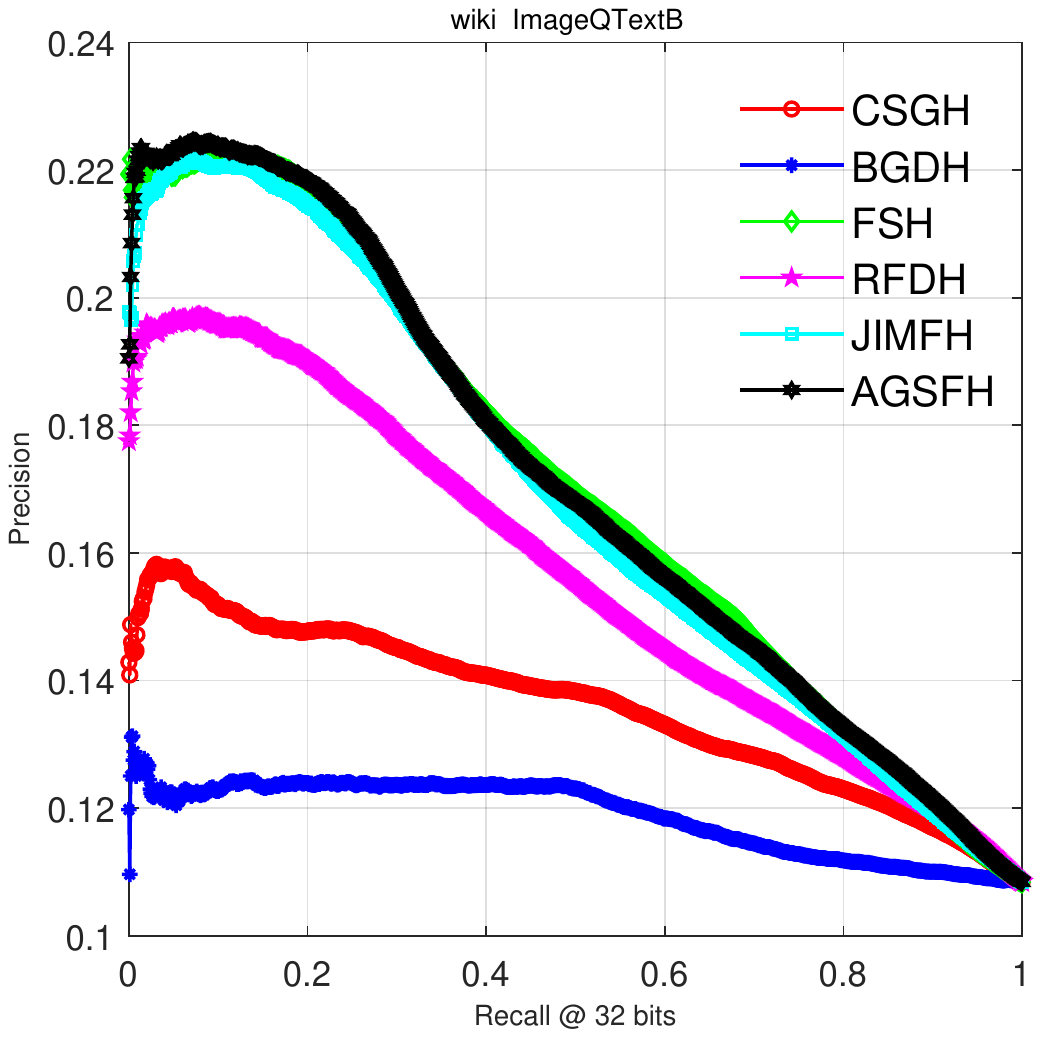}}
\hfill
\centering
\subfigure[ ]{
  \includegraphics[width=.3\textwidth]{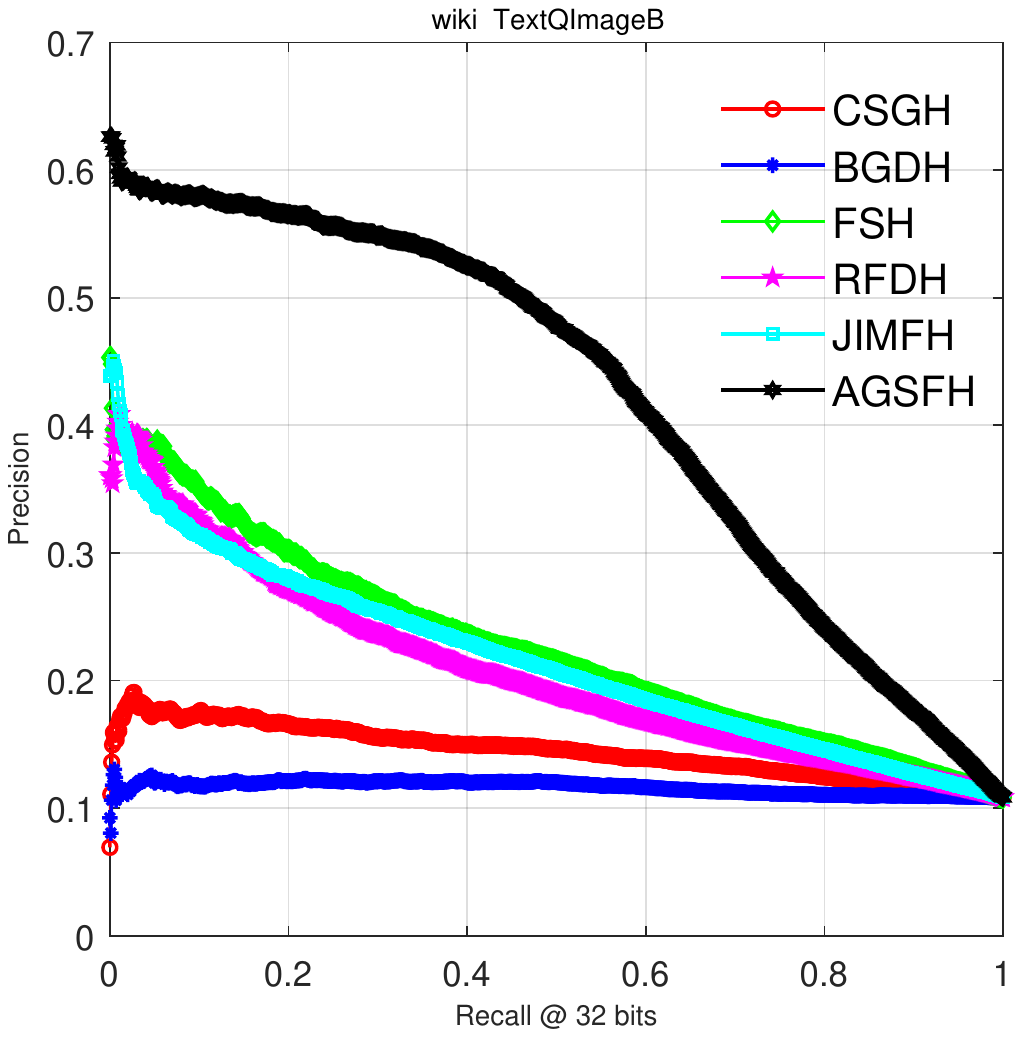}}
\caption{Precision-Recall Curves @ $32$ bits on three cross-modal benchmark datasets. (a) and (b) are Nuswide. (c) and (d) are MirFlickr25K. (e) and (f) are Wiki.}
\label{figure4}
\end{figure*}

\begin{figure*}[t]
\centering
\includegraphics[width=1.0\textwidth]{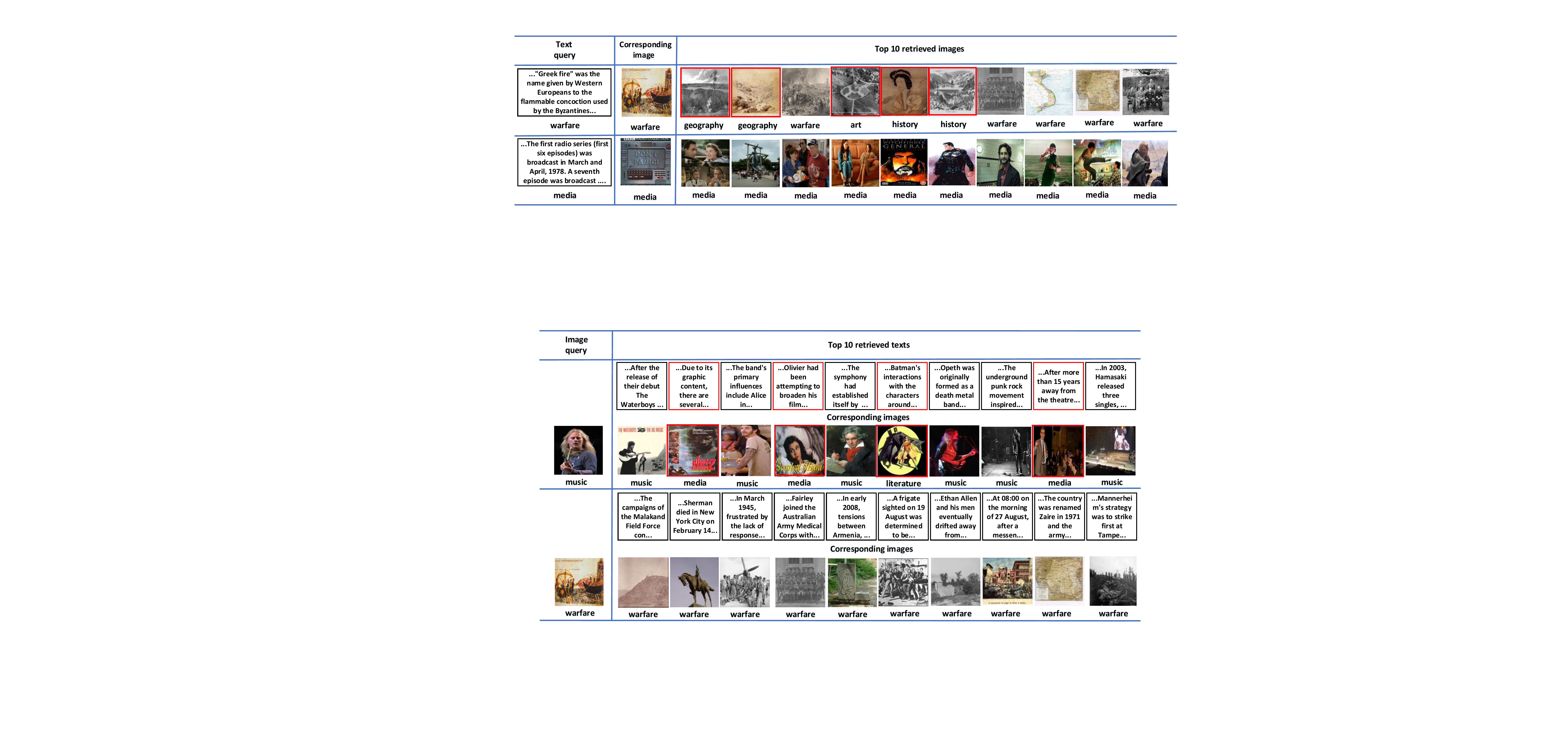}
\caption{Examples of image-based text retrieval on Wiki using AGSFH. For each image query (first column), we have obtained top $10$ retrieved texts (column $2$-$11$): top row demonstrates the retrieved texts, middle row demonstrates its corresponding images, bottom row demonstrates its corresponding class labels. The retrieved results irrelevant to the query in semantic categories are marked with the red box.}
\label{figure32}
\end{figure*}

\begin{figure*}[t]
\centering
\includegraphics[width=1.0\textwidth]{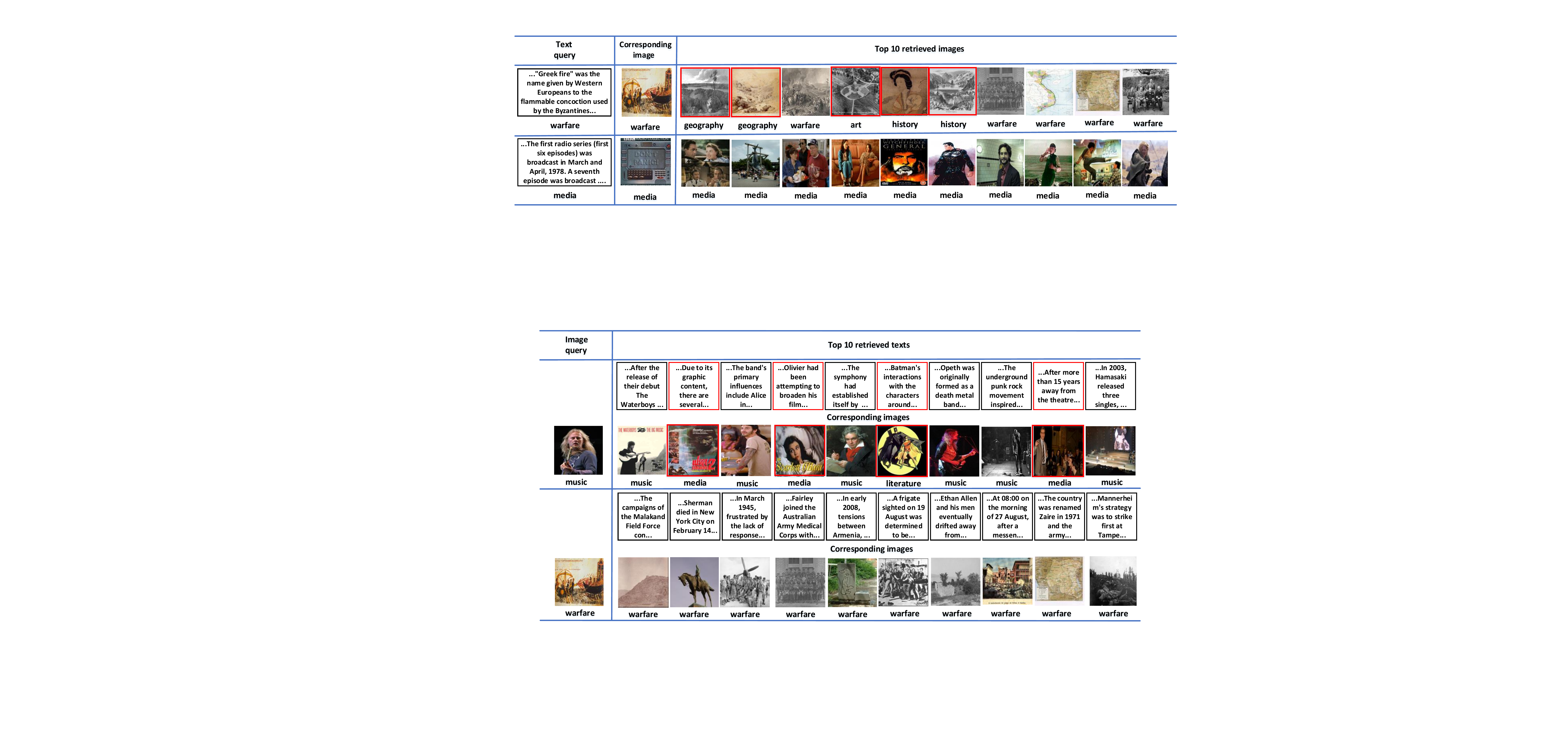}
\caption{Examples of text-based image retrieval on Wiki using AGSFH. For each text query (first column), the top $10$ retrieved images are shown in the third column: top row demonstrates the retrieved images, bottom row demonstrates its corresponding class labels. The second column shows the ground truth image corresponding to the text query. The retrieved results irrelevant to the query in semantic categories are marked with the red box.}
\label{figure31}
\end{figure*}

In this section, we show the comparison results of AGSFH with other state-of-the-art CMH approaches on three multi-modal benchmark databases. All the results are realized on a 64-bit Windows PC with a 2.20GHz i7-8750H CPU and 16.0GB RAM.

\subsection{Experimental Settings}

\subsubsection{Datasets}

We perform experiments on three multi-modal benchmarks datasets: Wiki \cite{jour9}, MIRFlickr25K \cite{jour10}, and NUS-WIDE \cite{jour11}. The details are given as follows.

Wiki \cite{jour9} dataset is comprised of $2,866$ image-text pairs, which are sampled from Wikipedia. The dataset contains ten semantic categories such that each pair belongs to one of these categories. Each image is represented by a $128$-dimensional bag-of-words visual vector, and each text is represented by a $10$-dimensional word-vectors. We randomly select $2,173$ image-text pairs from the original dataset as the training set (which is used as the database in the retrieval task), and the remaining $693$ image-text pairs are used as the query set.

MIRFlickr25K \cite{jour10} is a dataset containing $25,000$ image-text pairs in $24$ semantic categories. In our experiment, we remove pairs whose textual tags appear less than $20$ times. Accordingly, we get $20,015$ pairs in total, in which the $150$-dimensional histogram vector is used to represent the image, and the $500$-dimensional latent semantic vector is used to represent the text. We randomly select $2,000$ image-tag pairs as the query set and the remaining $18,015$ pairs as the database in the retrieval task. We also randomly selected $5000$ image-text pairs from the remaining $18,015$ pairs as the training set to learn the hash model.

NUS-WIDE \cite{jour11} is a multi-label image dataset that contains $269,648$ images with $5,018$ unique tags in $81$ categories. Similar to \cite{jour11}, we sample the ten largest categories with the corresponding $186,577$ images. We then randomly split this dataset into a training set with $184,577$ image-text pairs and a query set with $2,000$ image-text pairs. We should note that each image is annotated by a $500$-dimensional bag-of-words visual vector, and texts are represented by a $1,000$-dimensional index vector. Besides, we randomly selected $5000$ image-text pairs as the training set, which is used to learn the hash model.

\subsubsection{Compared Methods}

Our model is a shallow unsupervised CMH learning model. In order to conduct a fair comparison, we conduct several experiments and compare the AGSFH method with other shallow unsupervised CMH baselines: CSGH \cite{proceeding14}, BGDH \cite{proceeding16}, FSH \cite{proceeding17}, RFDH \cite{jour5}, JIMFH \cite{jour23}. To implement the baselines, we use and modify their provided source codes. We note that the results are averaged over five runs.

For further showing the superiority of performance of our proposal, we compare our work with some deep learning methods, i.e., DCMH \cite{ proceeding23}, DDCMH \cite{ proceeding24}, and DBRC \cite{jour14}, on MirFlickr25K in Section \ref{sec4.5}.

\subsubsection{Implementation Details}

We utilize the public codes with the fixed parameters in the corresponding papers to perform the baseline CMH methods. Besides, in our experiments, the hyper-parameter $\lambda$ is a trade-off hyper-parameter, which is set as $300$ on three datasets. The weight controller hyper-parameters $\gamma_1$ and $\gamma_3$ are both set to be $0.01$. The regularization hyper-parameter $\gamma_2$ is set to $10$ for better performance. The number of clusters $C$ is set as $60$ in all our experiments. For training efficiency, the number of anchor points $P$ is set to $900$ for all the datasets. The number of neighbor points $k$ is set to $45$.

\subsubsection{Evaluation Criteria}

We adopt three standard metrics to show the retrieval performance of AGSFH: mean average precision (MAP), topN-precision curves, and precision-recall curves. The MAP metric is calculated by the top $50$ returned retrieval samples. A more detailed introduction of these evaluation criteria is contained in \cite{proceeding30}. Two typical cross-modal retrieval tasks are evaluated for  our approach: image-query-text task (i.e. I$\rightarrow$T) and text-query-image task (i.e. T$\rightarrow$I).

\subsection{Experimental Analysis}
\renewcommand{\arraystretch}{1.0}
\begin{table*}[htb]
  \centering
  \tiny
  \setlength{\tabcolsep}{5pt}
  \caption{Training Time (in Seconds) of Different Hashing Methods on MIRFlickr25K and NUS-WIDE.}
  \label {Table.4}
  \resizebox{0.9\textwidth}{!}{
    \begin{tabular}{cccccccccc}
    \hline
    \multirow{2}{*}{Methods}
    &\multicolumn{4}{c}{MIRFlickr25K} &\multicolumn{4}{c}{NUS-WIDE}\cr
    \cmidrule(lr){2-5} \cmidrule(lr){6-9}
    &16 bits&32 bits&64 bits&128 bits&16 bits&32 bits&64 bits&128 bits\cr
    \hline
    CSGH &23.4576 &22.1885 &22.6214 &23.2040 &434.8411 &438.4742 &446.0972 &455.3075
\cr
    BGDH &0.6640 &0.7100 &1.0135 &1.4828 &8.8514 &10.0385 &12.6808 &16.6654
\cr
    FSH &5.2739 &5.2588 &5.6416 &6.2217 &75.6872 &77.9825 &83.4182 &89.5166
\cr
    RFDH &61.9264 &125.8577 &296.4683 &801.0725 &2132.1634 &4045.4016 &8995.0049 &22128.4188
\cr
    JIMFH &2.1706 &2.3945 &3.0242 &4.3020 &23.3772 &26.9233 &32.0088 &43.6795
\cr
    {\bf AGSFH} &79.2261 &80.3913 &136.8426 &79.4166 &638.8646 &871.4929 &794.8452 &961.3188
\cr
    \hline
    \end{tabular}}
\end{table*}

\begin{table*}[htb]
  \centering
  \tiny
  \setlength{\belowcaptionskip}{5pt}
  %\captionsetup{justification=centering}
  \caption{Testing Time (in Seconds) of Different Hashing Methods on MIRFlickr25K and NUS-WIDE.}
  \label {Table.5}
  \resizebox{0.9\textwidth}{!}{
    \begin{tabular}{ccccccccccc}
    \hline
    \multirow{2}{*}{Tasks}& \multirow{2}{*}{Methods} &
    \multicolumn{4}{c}{MIRFlickr25K} &\multicolumn{4}{c}{NUS-WIDE}\cr
    \cmidrule(lr){3-6} \cmidrule(lr){7-10}
    &&16 bits&32 bits&64 bits&128 bits&16 bits&32 bits&64 bits&128 bits\cr
    \hline
    \multirow{6}{*}{I$\rightarrow$T}
    &CSGH &0.1935 &0.2320 &0.3626 &1.0183 &3.2961 &4.4186 &6.8955 &12.4145
\cr
    &BGDH &0.1662 &0.2322 &0.3610 &1.0342 &3.3980 &4.4255 &6.9406 &12.4911
\cr
    &FSH &0.1984 &0.2464 &0.3555 &1.0110 &3.4221 &4.5401 &7.1110 &12.6618
\cr
    &RFDH &0.1941 &0.2522 &0.3672 &1.0303 &3.4753 &4.4805 &6.9003 &12.6593
\cr
    &JIMFH &0.1782 &0.2405 &0.3742 &1.0076 &3.2559 &4.5116 &7.0505 &12.6736
\cr
    &{\bf AGSFH} &0.1528 &0.2403 &0.3784 &0.9411 &3.9483 &4.5114 &6.9818 &12.6628
\cr
  \hline
    \multirow{6}{*}{T$\rightarrow$I}
    &CSGH &0.1562 &0.2133 &0.3670 &1.0157 &3.2620 &4.4325 &6.9195 &12.3532
\cr
    &BGDH &0.1442 &0.2194 &0.3464 &1.0355 &3.4354 &4.4582 &6.8954 &12.3665
\cr
    &FSH &0.1533 &0.2130 &0.3639 &1.0101 &3.4172 &4.4222 &6.9067 &12.422
\cr
    &RFDH &0.1455 &0.2263 &0.3584 &1.0281 &3.3462 &4.4351 &6.9336 &12.4925
\cr
    &JIMFH &0.1527 &0.2101 &0.3444 &1.0105 &3.4214 &4.4370 &6.9354 &12.4270
\cr
    &{\bf AGSFH} &0.1371 &0.2224 &0.3532 &0.9485 &3.4478 &4.4860 &6.9820 &12.4307
\cr
    \hline
    \end{tabular}}
\end{table*}

\subsubsection{Retrieval Performance}

In Table.\ref{Table.1}, \ref{Table.2} and \ref{Table.3}, the MAP evaluation results are exhibited on all three data sets, i.e., Wiki, MIRFlickr25K, and NUS-WIDE. From these figures, for all cross-modal tasks (i.e., image-query-text and text-query-image), AGSFH achieves a significantly better result than all comparison methods On Wiki and MIRFlickr25K. Besides, on NUS-WIDE, AGSFH also achieves comparable performance with JIMFH and FSH on image-query-text task, outperforming other remaining comparison methods, and shows significantly better performance than all comparison methods on text-query-image task. The superiority of AGSFH can be attributed to their capability to reduce the effect of information loss, which directly learns the intrinsic anchor graph to exploit the geometric property of underlying data structure across multiple modalities, as well as avoids the large quantization error. Besides, AGSFH learns the structure of the intrinsic anchor graph adaptively to make training instances clustered into semantic space and preserves the anchor fusion affinity into the common binary Hamming space to guarantee that binary data can preserve the semantic relationship of training instances in semantic space. The above observations show the effectiveness of the proposed AGSFH. We can find another observation that the average increase of text-query-image retrieval task is larger than the average increase of image-query-text retrieval task. This is because that image includes more noise and outliers than text.

The topN-precision curves with code length $32$ bits on all three data sets are demonstrated in Figs. \ref{figure3}. From the experimental results, the topN-precision results are in accordance with mAP evaluation values. AGSFH has better performance than others comparison methods for cross-modal hashing search tasks on Wiki and MIRFlickr25K. Furthermore, on NUS-WIDE, AGSFH demonstrates comparable performance with JIMFH and FSH on image-query-text task, outperforming other remaining comparison methods, and shows significantly better performance than all comparison methods on text-query-image task. In the retrieval system, we focus more on the front items in the retrieved list returned by the search algorithm. Hence, AGSFH achieves better performance on all retrieval tasks in some sense.

From Table.\ref{Table.1}, \ref{Table.2}, \ref{Table.3} and Figs. \ref{figure3}-\ref{figure4}, AGSFH usually demonstrates large margins on performance when compared with other methods about cross-modal hashing search tasks on Wiki and MIRFlickr25K. At the same time, on NUS-WIDE, AGSFH also exhibits comparable performance with JIMFH and FSH on image-query-text task, better than other remaining comparison methods, and shows significantly better performance than all comparison methods on text-query-image task. We consider two possible reasons for explaining this phenomenon. Firstly, AGSFH directly learns the intrinsic anchor graph to exploit the geometric property of underlying data structure across multiple modalities, as well as avoids the large quantization error, which can be robust to data (i.e., image and text data) outliers and noises. Thus, AGSFH can achieve performance improvement. Secondly, AGSFH adjusts the structure of the intrinsic anchor graph adaptively to make training instances clustered into semantic space and preserves the anchor fusion affinity into the common binary Hamming space. This process can extract the high-level hidden semantic relationship in the image and text.  Therefore, AGSFH could find the common semantic clusters that reflect the semantic properties more precisely. On the consequences, under the guidance of the common semantic clusters, AGSFH can achieve better performance on cross-modal retrieval tasks.

The precision-recall curves with the code length of $32$ bits are also demonstrated in Fig. \ref{figure4}. By calculating the area under precision-recall curves, we can discover that AGSFH outperforms comparison methods for cross-modal hashing search tasks on Wiki and MIRFlickr25K. In addition, on NUS-WIDE, AGSFH has comparable performance with JIMFH and FSH on image-query-text task, better than other remaining comparison methods, and shows significantly better performance than all comparison methods on text-query-image task.

Moreover, two examples of image-query-text retrieval task and text-query-image retrieval task over Wiki dataset by AGSFH are performed. The results are showed in Fig. \ref{figure32} and Fig. \ref{figure31}. Fig.\ref{figure32} reports that the classes of the second, fourth, sixth, and ninth searched texts are different from the image query belonging to the music category (in the first row). This is because the visual appearance of the image query in the music category is very similar to the incorrect images in the retrieved results. The example of text retrieving images is displayed in Fig.\ref{figure31}. As we can see, the classes of the first, second, fourth, fifth, and sixth retrieved images by AGSFH are different from the text query belonging to the warfare category. We can find that the visual appearance of the text query's corresponding images is very similar to the incorrect retrieved images.

\subsubsection{Training Time}

We check the time complexity of all the compared CMH methods in the training phase on two large benchmark datasets MIRFlickr25K and NUSWIDE, showing the efficiency of our approach AGSFH. Table. \ref{Table.4} illustrates the training time comparison result. From the table, compared with CSGH, BGDH, FSH, and JIMFH, we can find our AGSFH takes a little longer time than these hashing methods but spends a shorter time to RFDH with different code lengths from $16$ bits to $128$ bits. On the other hand, our approach AGSFH has nearly no change in acceptable training time with different code lengths from $16$ bits to $128$ bits, showing the superiority of AGSFH in time complexity.

\subsubsection{Testing Time}

We compare the testing time with all the compared CMH methods in Table. \ref{Table.5}. From this table, we observe that the compared cross-modal hashing methods take nearly identical time for retrieval as well as our AGSFH.

\begin{table}[htb]
%\tiny
  \centering
  \tiny
  \setlength{\belowcaptionskip}{5pt}
  %\captionsetup{justification=centering}
  \caption{MAP Comparison of AGSFH and Deep Cross-modal Hashing Methods on MirFlickr25K.}
  \label {Table.6}
  \resizebox{0.48\textwidth}{!}{
    \begin{tabular}{cccccc}
    \hline
    \multirow{2}{*}{Tasks}& \multirow{2}{*}{Methods} &
    \multicolumn{3}{c}{MirFlickr25K} \cr\cline{3-5}
    &&16 bits&32 bits&64 bits\cr
    \hline
    \multirow{4}{*}{I$\rightarrow$T}
    &DCMH &0.7410 &0.7465 &0.7485
\cr
   &DDCMH &{\bf 0.8208} &{\bf 0.8434} &{\bf 0.8551}
\cr
    &DBRC &0.5922 &0.5922 &0.5854
\cr
    &{\bf AGSFH} &0.7239 &0.7738 &0.8135

\cr
\hline
    \multirow{4}{*}{T$\rightarrow$I}
    &DCMH &{\bf 0.7827} &0.7900 &0.7932
\cr
   &DDCMH &0.7731 &0.7766 &0.7905
\cr
    &DBRC &0.5938 &0.5952 &0.5938
\cr
    &{\bf AGSFH} &0.7505 &{\bf 0.8009} &{\bf 0.8261}
\cr
    \hline
    \end{tabular}}
\end{table}

\subsubsection{Comparison with Deep Cross-modal Hashing} \label{sec4.5}

We further compare AGSFH with three state-of-the-art deep CMH methods, i.e., DCMH \cite{ proceeding23}, DDCMH \cite{ proceeding24}, and DBRC \cite{jour14}, on MirFlickr25K dataset. We utilize the deep image features, extracted by the CNN-F network \cite{ proceeding31} with the same parameters in \cite{ proceeding23} and the original text features for evaluating the MAP results of AGSFH. In addition, the performance of DCMH, DDCMH, and DBRC (without public code) is presented using the results in its paper.

The results of the experiments are shown in Table. \ref{Table.6}. From Table. \ref{Table.6}, we could find that AGSFH outperforms DCMH and DBRC on both the Image-to-Text and Text-to-Image tasks, but AGSFH is inferior to DDCMH on the Image-to-Text task and is better than it on the Text-to-Image task. From these observations, we conduct the conclusion that AGSFH is not a deep hashing model, yet it can outperform some of the state-of-the-art deep cross-modal hashing methods, i.e., DCMH, DBRC. As the DDCMH shows significantly superior results than the proposed AGSFH, our AGSFH contains many strengths over DDCMH, like simple structure, easy to implement, time-efficient, strong interpretability, and so on. It further verifies the effectiveness of the proposed model AGSFH.

\subsection{Empirical Analysis}
\begin{figure}[ht]
\centering
\includegraphics[width=0.5\columnwidth]{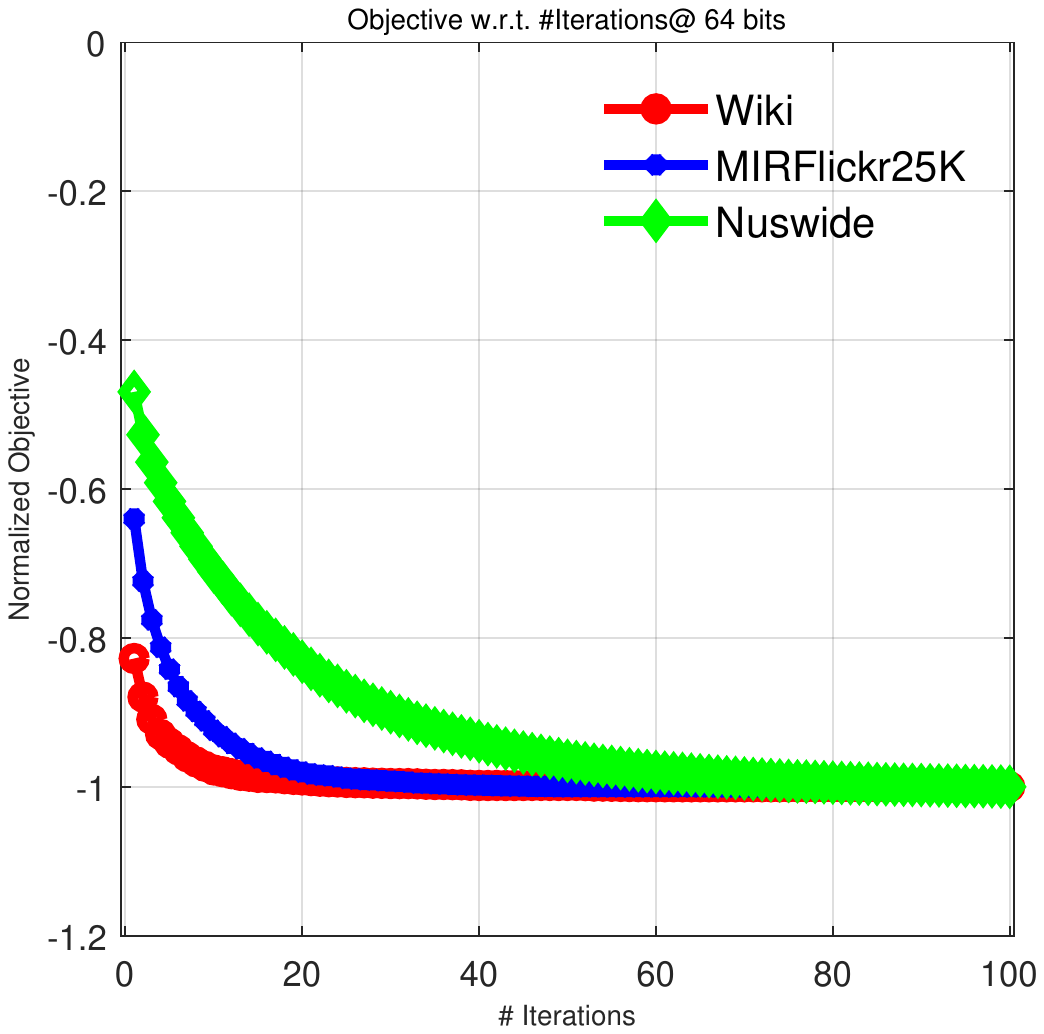}
\caption{Convergence analysis.}
\label{figure5}
\end{figure}

\begin{figure*}[ht]
\centering
\subfigure[ ]{
  \includegraphics[width=0.45\columnwidth]{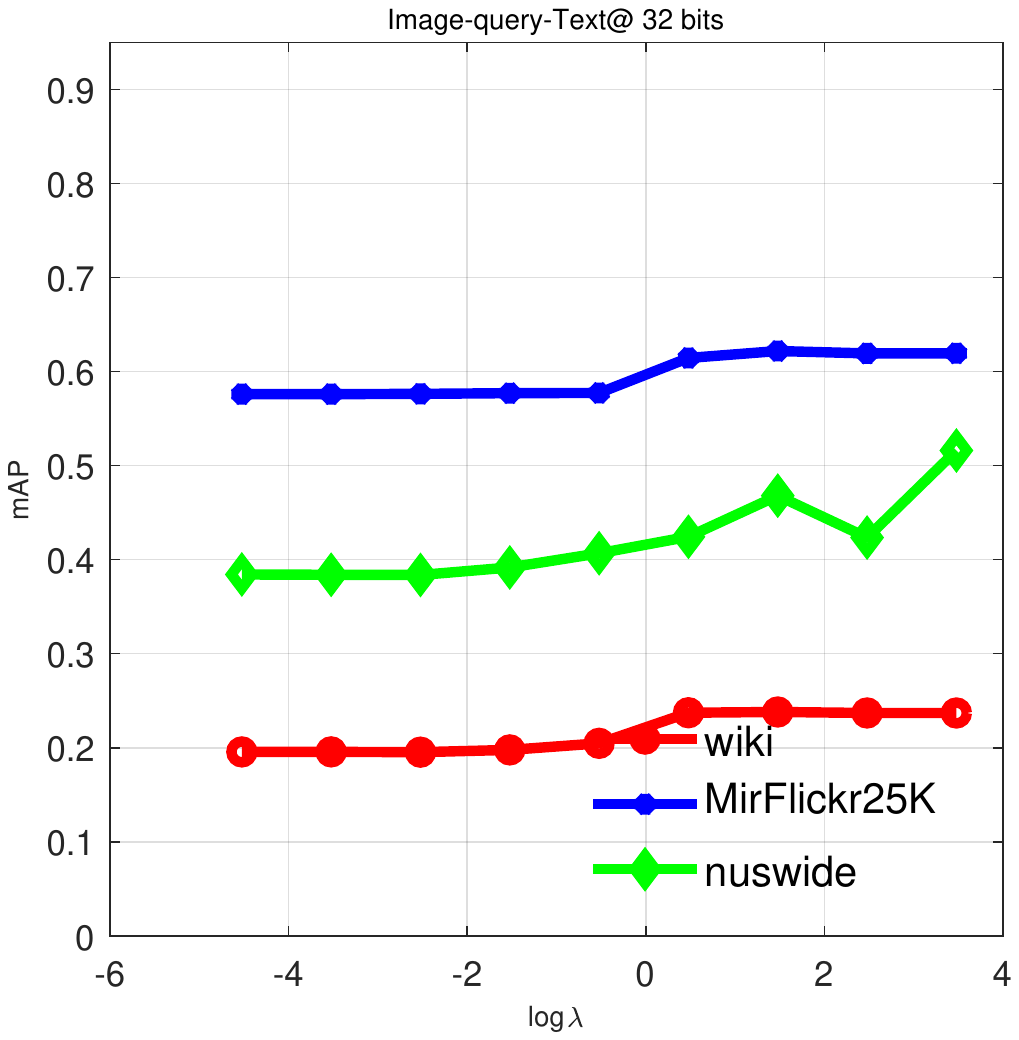}}
\centering
\subfigure[ ]{
  \includegraphics[width=0.45\columnwidth]{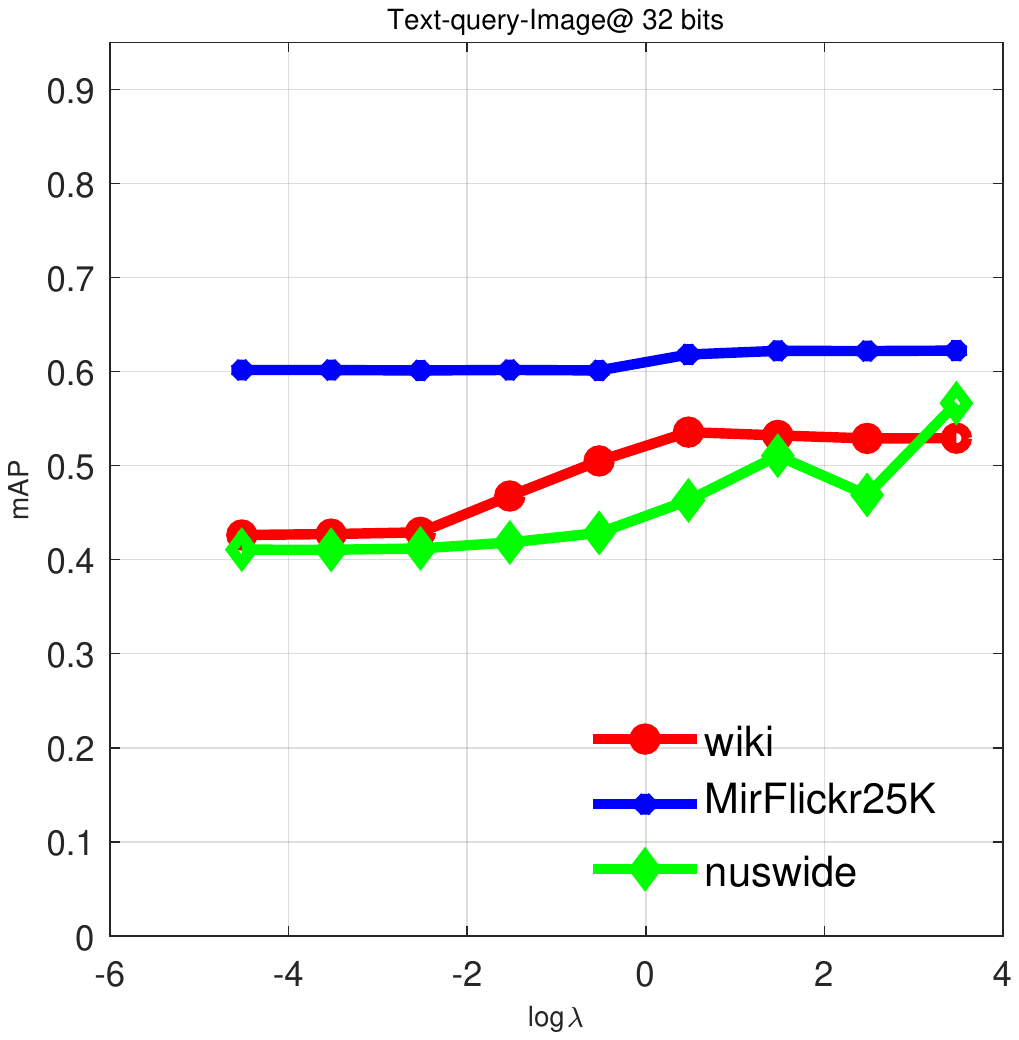}}
\centering
\subfigure[ ]{
  \includegraphics[width=0.45\columnwidth]{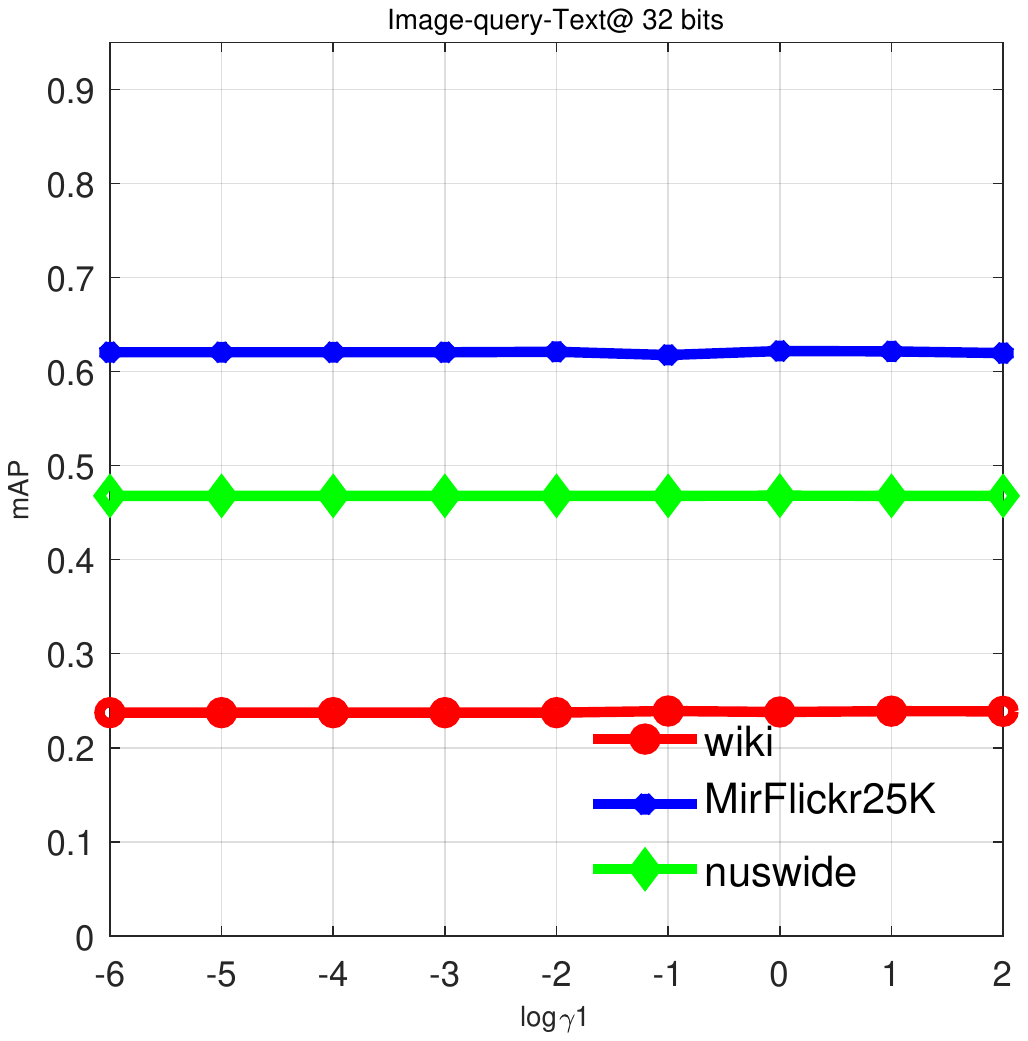}}
\centering
\subfigure[ ]{
  \includegraphics[width=0.45\columnwidth]{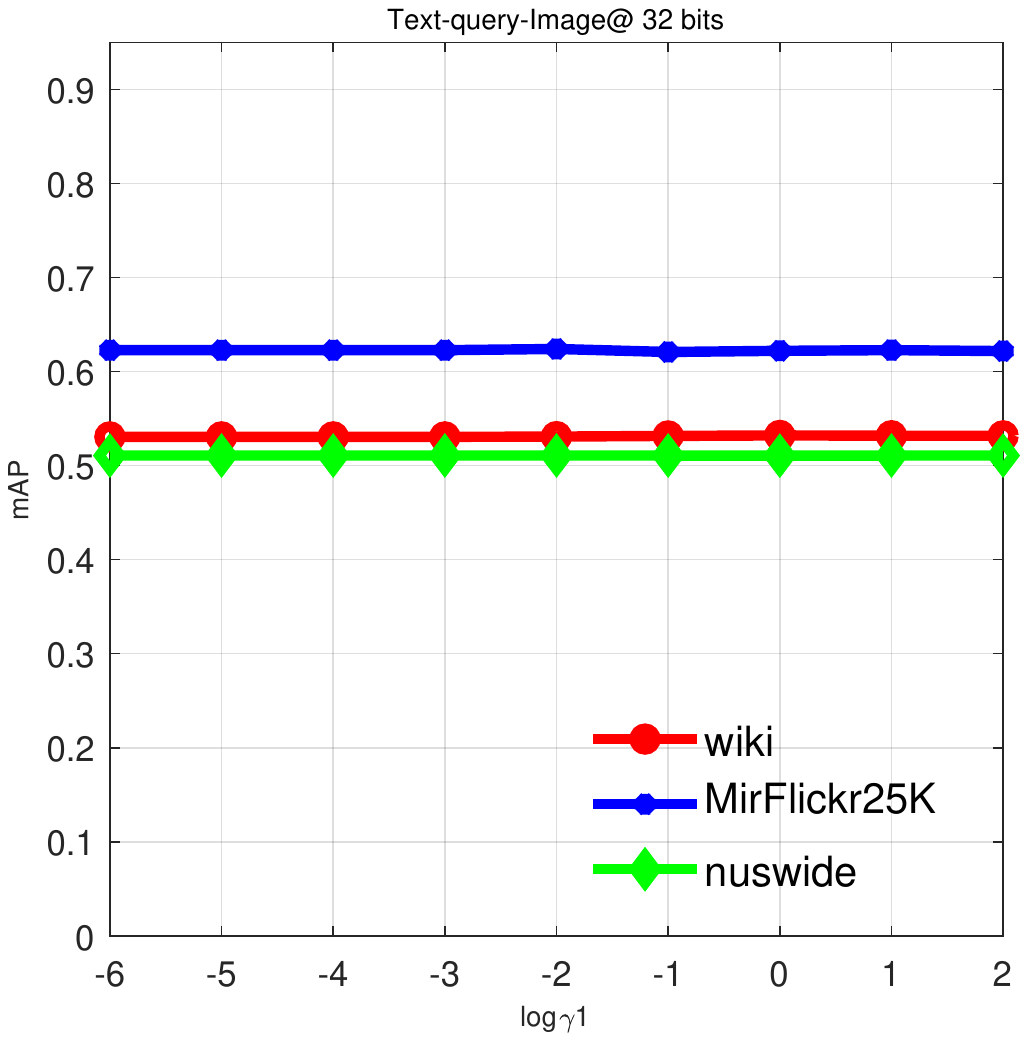}}
\caption{MAP values versus hyper-parameters. (a) and (b) are $\lambda$. (c) and (d) are $\gamma_1$.}
\label{figure6}
\end{figure*}

\begin{figure*}[ht]
\centering
\subfigure[ ]{
  \includegraphics[width=0.45\columnwidth]{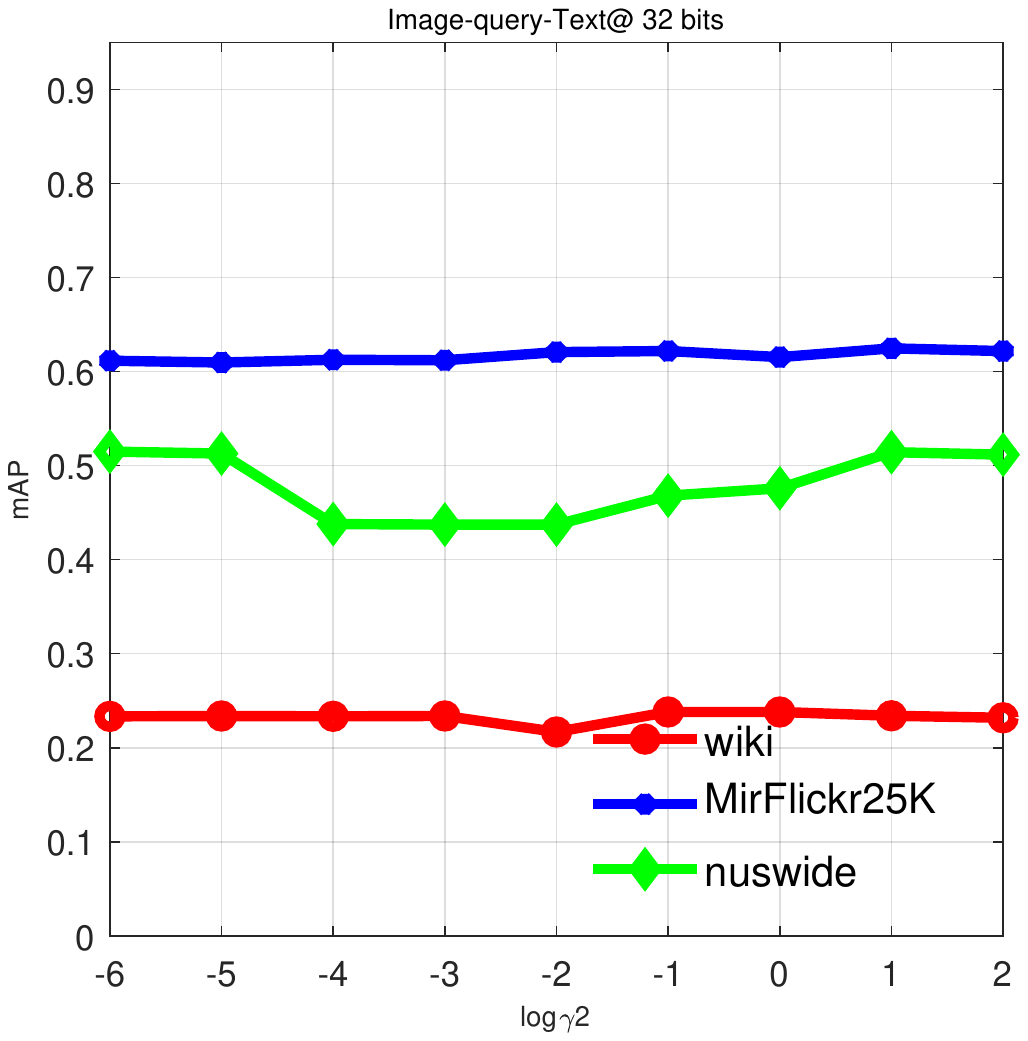}}
\centering
\subfigure[ ]{
  \includegraphics[width=0.45\columnwidth]{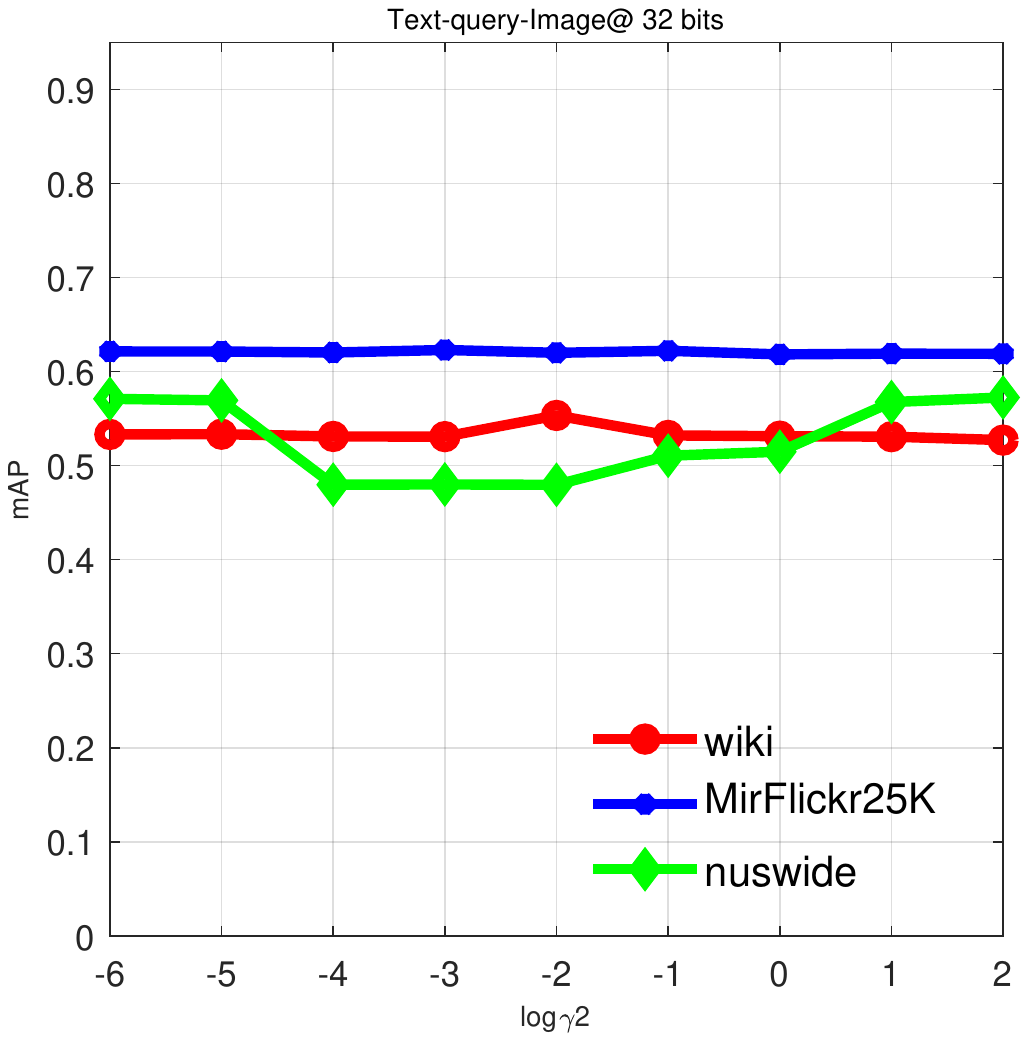}}
\centering
\subfigure[ ]{
  \includegraphics[width=0.45\columnwidth]{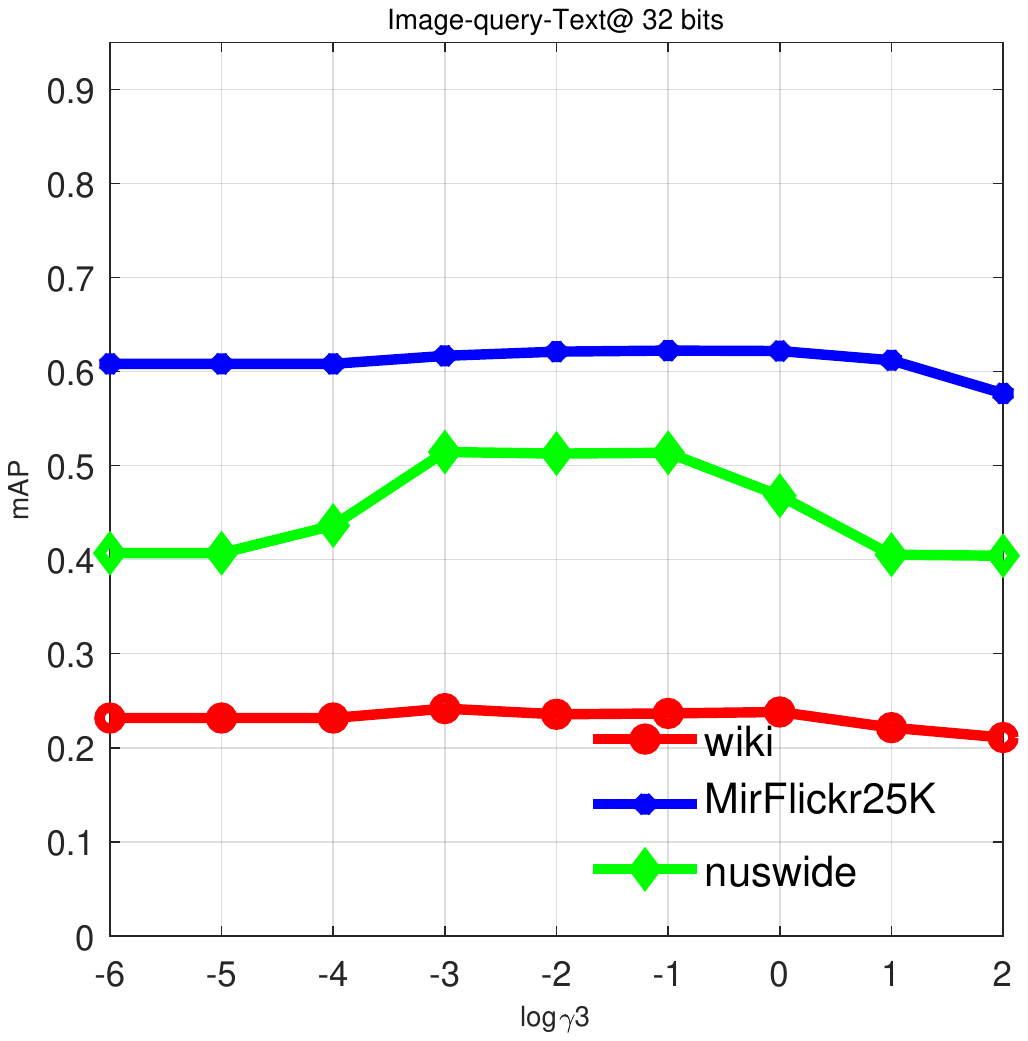}}
\centering
\subfigure[ ]{
  \includegraphics[width=0.45\columnwidth]{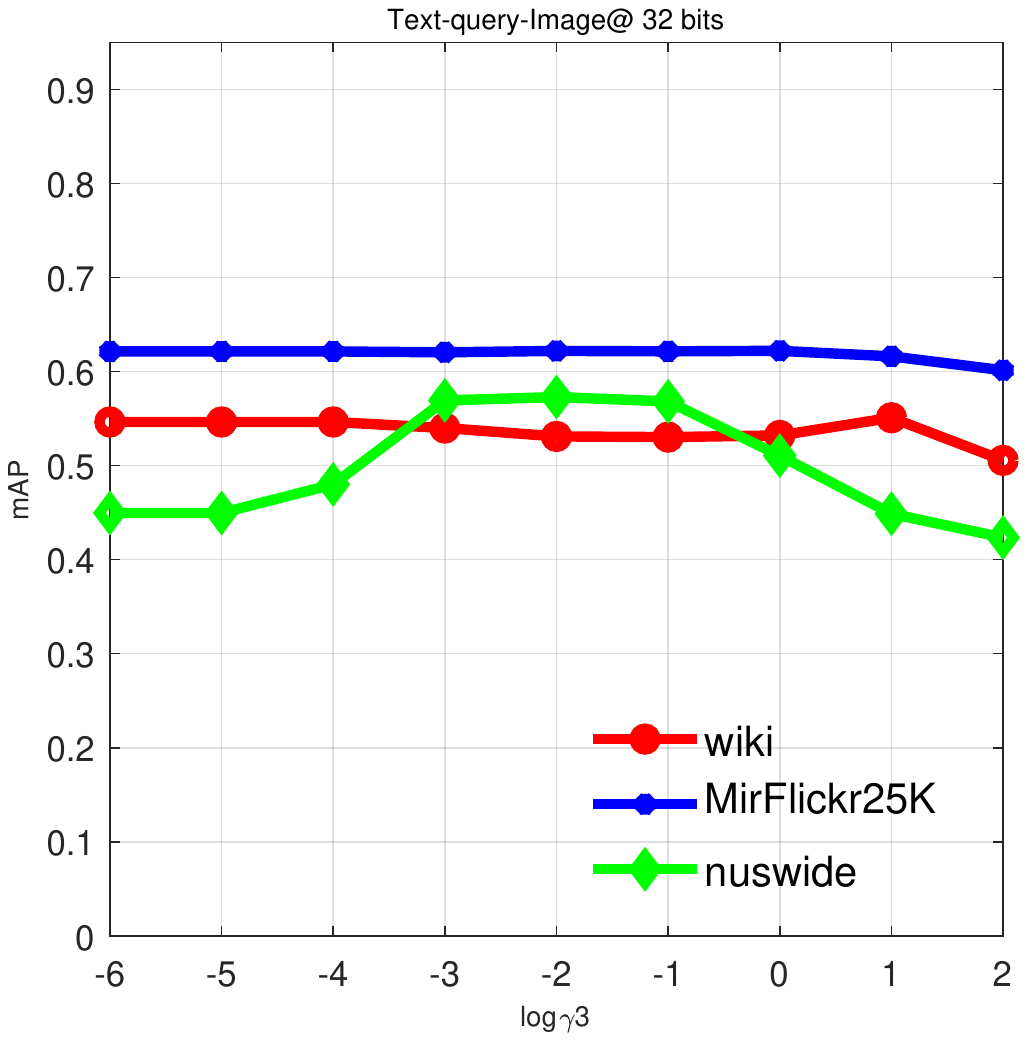}}
\caption{MAP values versus hyper-parameters. (a) and (b) are $\gamma_2$. (c) and (d) are $\gamma_3$.}
\label{figure7}
\end{figure*}

\begin{figure*}[ht]
\centering
\subfigure[ ]{
  \includegraphics[width=0.45\columnwidth]{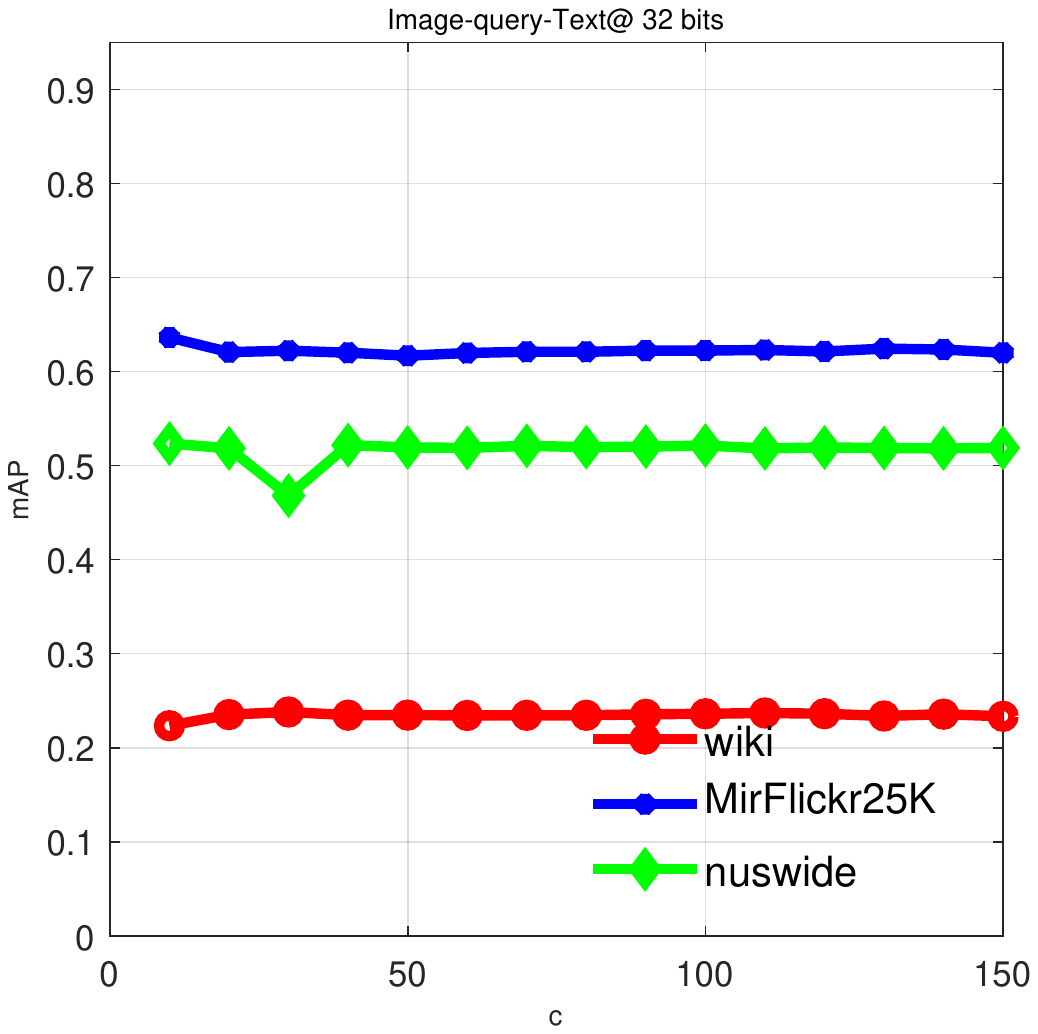}}
\centering
\subfigure[ ]{
  \includegraphics[width=0.45\columnwidth]{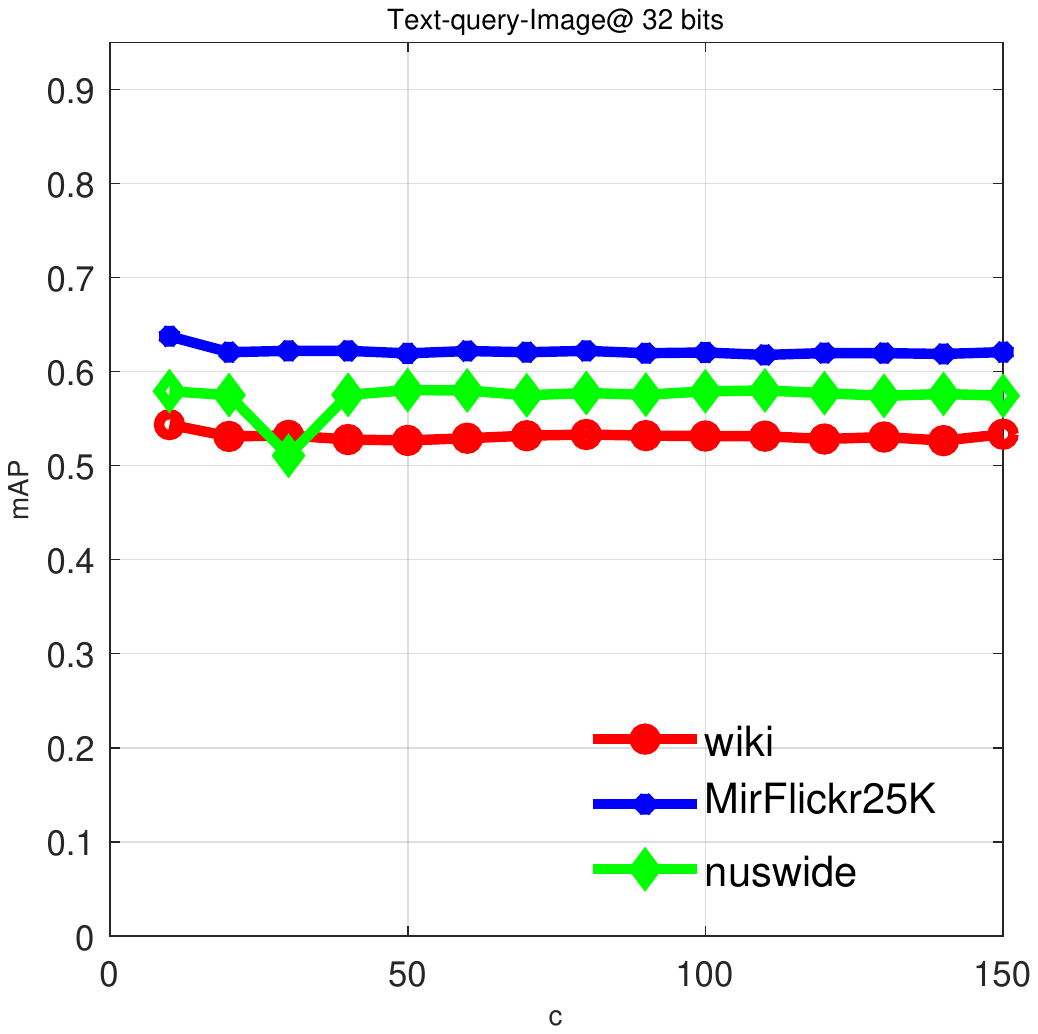}}
\centering
\subfigure[ ]{
  \includegraphics[width=0.45\columnwidth]{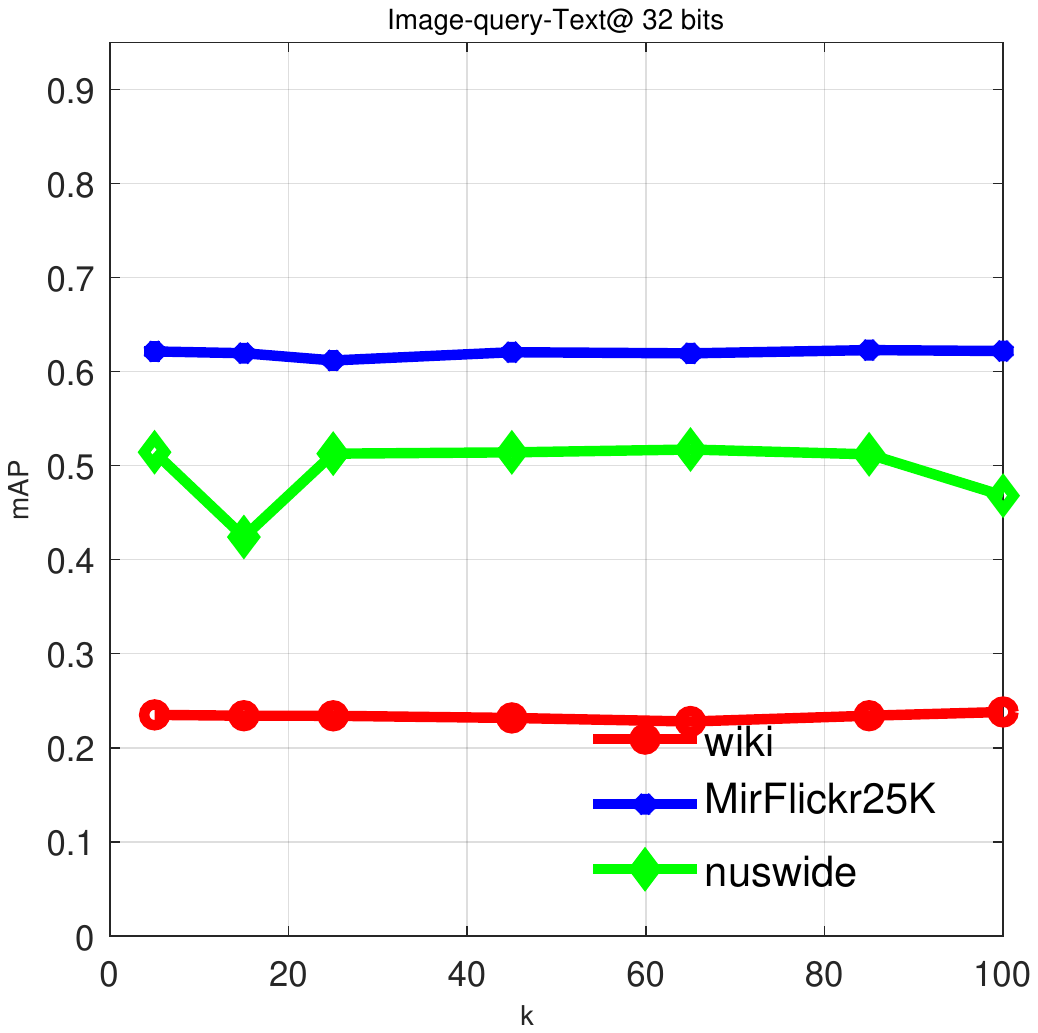}}
\centering
\subfigure[ ]{
  \includegraphics[width=0.45\columnwidth]{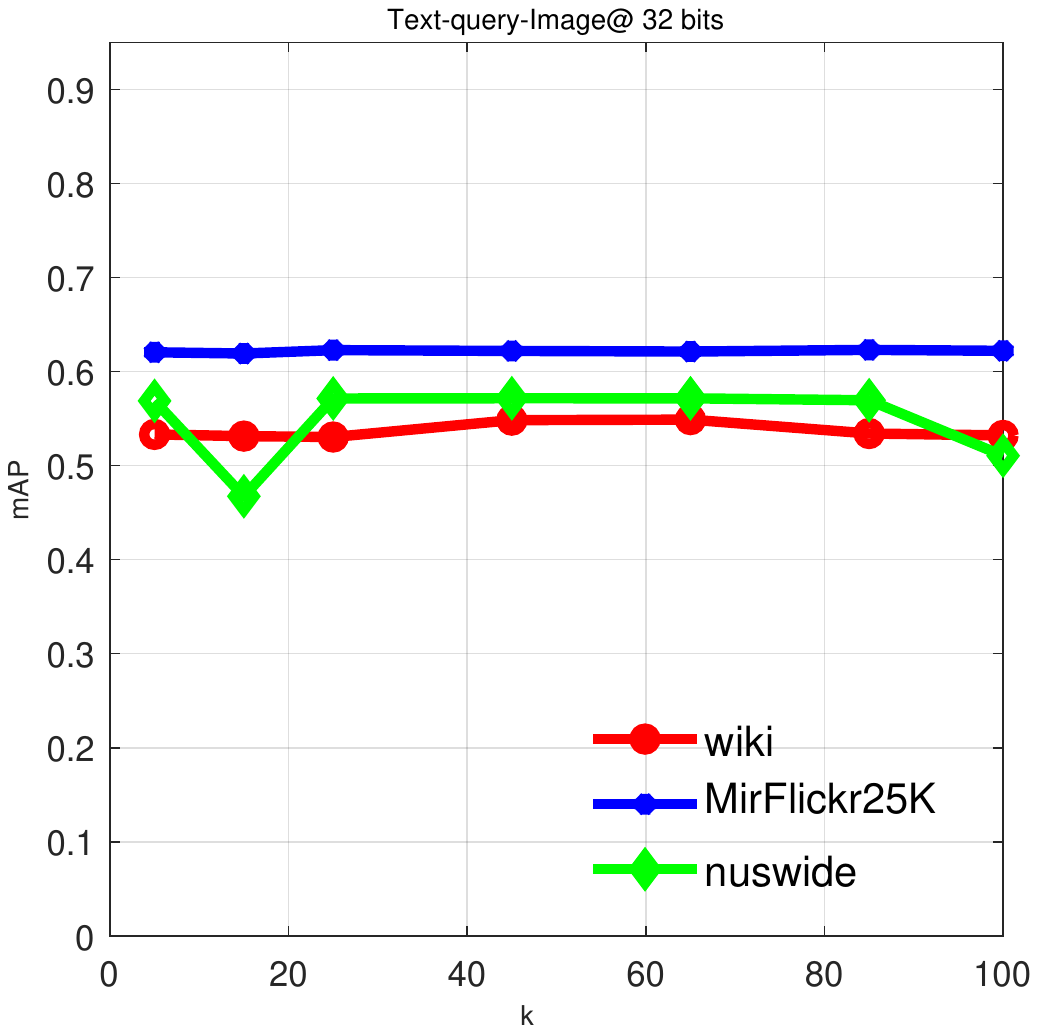}}
\caption{MAP values versus hyper-parameters. (a) and (b) are $c$. (c) and (d) are $k$.}
\label{figure8}
\end{figure*}

\begin{figure}[ht]
\centering
\subfigure{
  \includegraphics[width=0.45\columnwidth]{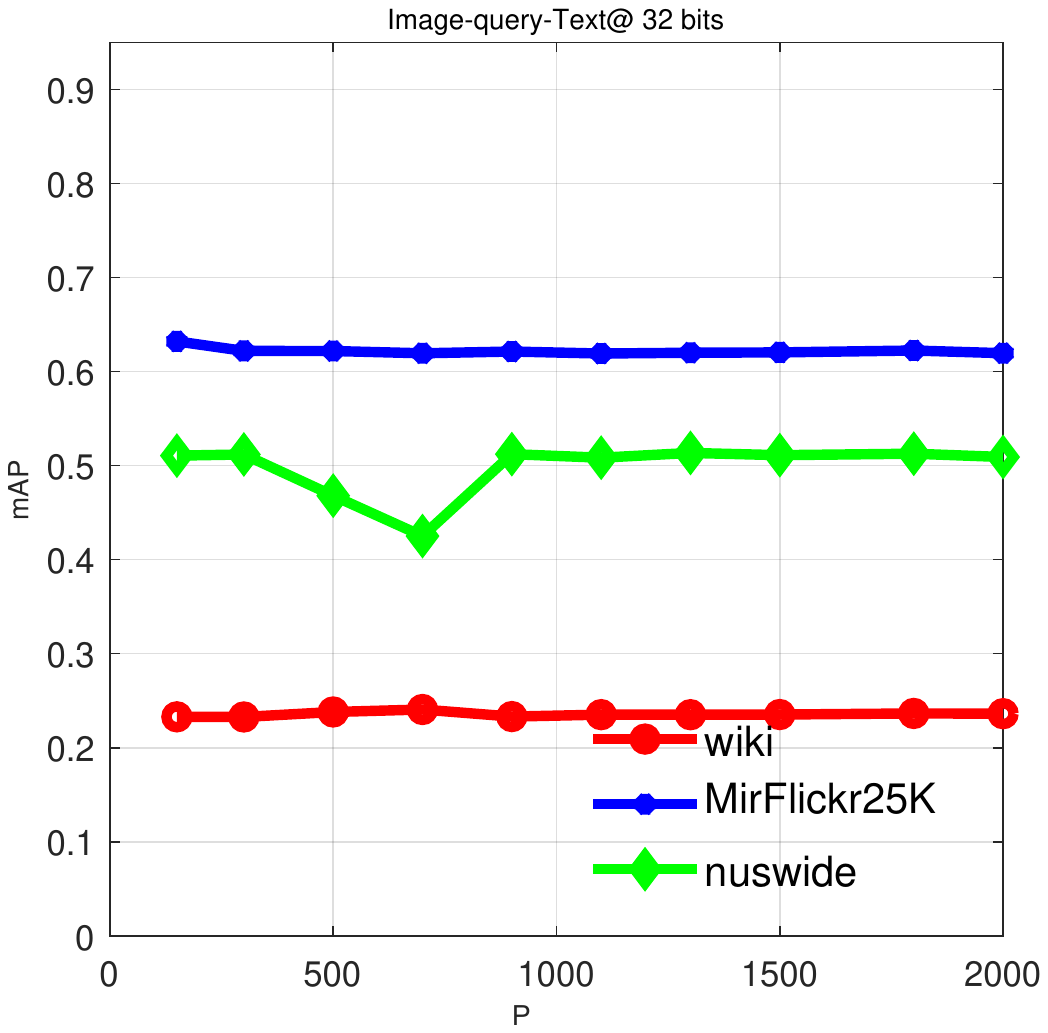}}
\centering
\subfigure{
  \includegraphics[width=0.45\columnwidth]{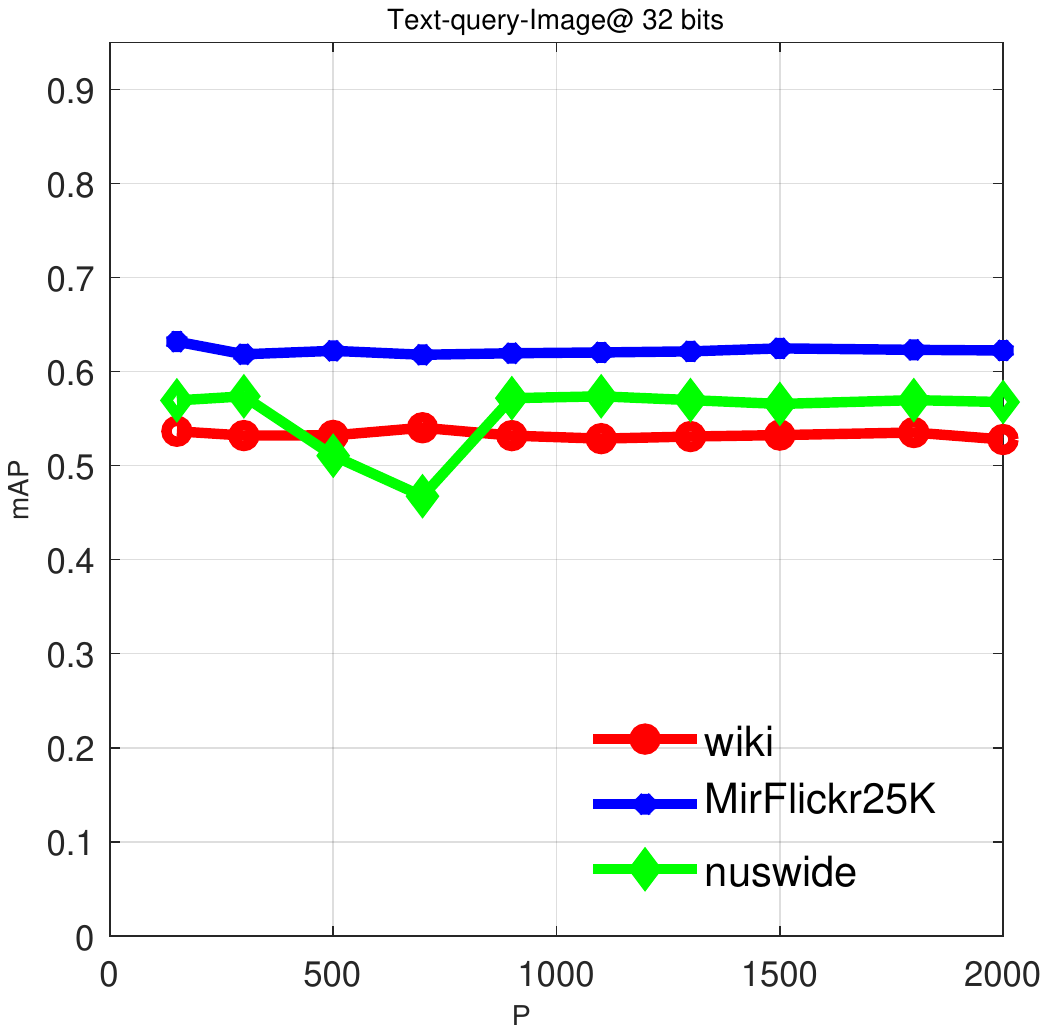}}
\caption{MAP values versus hyper-parameter $P$.}
\label{figure9}
\end{figure}

\subsubsection{Convergence} \label{Convergences}

We empirically validate the convergence property of the proposed AGSFH. Specifically, we conduct the empirical analysis of the value for the objective on all three datasets with the fixed $64$ bits hash code. Fig. \ref{figure5} illustrates the convergence curves, where we have normalized the value of the objective by the number of training data and the maximum value for the objective. From the figure Fig. \ref{figure5}, we can easily find that the objective value of AGSFH decreases sharply with less than $40$ iterations on all three datasets and has no significant change after several iterations. The result shows the fast convergence property of Algorithm \ref{alg:GSFH}.

\subsubsection{Ablation Experiments Analysis}
\begin{table}[htb]
%\tiny
  \centering
  \tiny
  \setlength{\belowcaptionskip}{5pt}
  %\captionsetup{justification=centering}
  \caption{Ablation Results of AGSFH on MirFlickr25K.}
  \label {Table.7}
  \resizebox{0.48\textwidth}{!}{
    \begin{tabular}{ccccccc}
    \hline
    \multirow{2}{*}{Tasks}& \multirow{2}{*}{Methods} &
    \multicolumn{3}{c}{MirFlickr25K} \cr\cline{3-6}
    &&16 bits&32 bits&64 bits&128 bits\cr
    \hline
    \multirow{3}{*}{I$\rightarrow$T}
    &AGSFH\_appro &0.6249	&{\bf 0.6531}	&0.6660	&0.6688
\cr
   &AGSFH\_reg &0.6397	&0.6353	&0.6609	&{\bf 0.6709}
\cr
    &{\bf AGSFH} &{\bf 0.6421}	&0.6509	&{\bf 0.6702}	&0.6701
\cr
\hline
    \multirow{3}{*}{T$\rightarrow$I}
    &AGSFH\_appro &0.6419	&0.6720	&0.6990	&0.7095
\cr
   &AGSFH\_reg &0.6496	&0.6612	&0.7013	&0.7026
\cr
    &{\bf AGSFH} &{\bf 0.6676}	&{\bf 0.6745}	&{\bf 0.7101}	&{\bf 0.7163}
\cr
    \hline
    \end{tabular}}
\end{table}

To obtain a deep insight about AGSFH, we further performed ablation experiments analysis on MIRFlickr-25K dataset, including two variations of AGSFH, that is: $1)$ AGSFH\_appro and $2)$ AGSFH\_reg. Compared with AGSFH, AGSFH\_appro drops out the second term in Eq.(\ref{GSFH19}). This second term means that the learned adaptive anchor graph $\hat{S}$ should best approximate the fused anchor graph $\hat{A}$. In contrast, AGSFH\_reg discards the third term in Eq.(\ref{GSFH19}). The third term means adding $L_2$-norm regularization to smooth the elements of the learned anchor graph $\hat{S}$. The MAP results of AGSFH and its variations are shown in Table. \ref{Table.7}. From this table, we can find the following.

\begin{enumerate}[(1)]
\setlength{\listparindent}{2em}

  \item AGSFH outperforms AGSFH\_appro and AGSFH\_reg greatly, showing the effect of the second term and the third term in Eq.(\ref{GSFH19}).
  \item AGSFH\_appro has approximately equivalent performance over AGSFH\_reg. This phenomenon indicates that the second term and the third term in Eq.(\ref{GSFH19}) are both useful for performance improvements. They have different effect improvements of AGSFH.
\end{enumerate}

\subsubsection{Parameter Sensitivity Analysis}

In this part, we analyze the parameter sensitivity of two cross-modal retrieval tasks with various experiment settings in empirical experiments on all datasets. Our AGSFH involves seven model hyper-parameters: the trade-off hyper-parameter $\lambda$, the weight controller hyper-parameters $\gamma_1$ and $\gamma_3$, the regularization hyper-parameter $\gamma_2$, the number of clusters $C$, the number of anchors points $P$, and the number of neighbor points $k$.

We fix the length of binary codes as $32$ bits and fix one of these two hyper-parameters with others fixed to verify the superior and stability of our approach AGSFH on a wide parameter range. Fig. \ref{figure6}, \ref{figure7}, \ref{figure8}, \ref{figure9} presents the MAP experimental results of AGSFH. From these figures, relatively stable and superior performance to a large range of parameter variations is shown, verifying the robustness to some parameter variations.

\section{Conclusion}
In this paper, we propose a novel cross-modal hashing approach for efficient retrieval tasks across modalities, termed Anchor Graph Structure Fusion Hashing (AGSFH). AGSFH directly learns an intrinsic anchor fusion graph, where the structure of the intrinsic anchor graph is adaptively tuned so that the number of components of the intrinsic graph is exactly equal to the number of clusters. Based on this process, training instances can be clustered into semantic space. Besides, AGSFH attempts to directly preserve the anchor fusion affinity with complementary information among multi-modalities data into the common binary Hamming space, capturing intrinsic similarity and structure across modalities by hash codes. A discrete optimization framework is designed to learn the unified binary codes across modalities. Extensive experimental results on three public social datasets demonstrate the superiority of AGSFH in cross-modal retrieval.

\bibliographystyle{IEEEtran}
\bibliography{mybibfile}

\end{document}